\documentclass[11pt]{article}
\usepackage[margin=1in]{geometry}
\usepackage{authblk}

\usepackage[utf8]{inputenc}
\usepackage[T1]{fontenc}
\usepackage{hyperref}
\usepackage{url}
\usepackage{booktabs}
\usepackage{amsfonts}
\usepackage{nicefrac}
\usepackage{microtype}
\usepackage{xcolor}

\usepackage{graphicx}
\usepackage{algorithm}
\usepackage{algorithmic}

\usepackage{amsmath}

\usepackage{amssymb}
\usepackage{mathtools}
\usepackage{amsthm}

\definecolor{algcomment}{RGB}{100,100,160}

\usepackage{enumitem}
\usepackage{multirow}
\usepackage{wrapfig}

\usepackage{natbib}

\usepackage[capitalize,noabbrev]{cleveref}

\theoremstyle{plain}
\newtheorem{theorem}{Theorem}[section]
\newtheorem{proposition}[theorem]{Proposition}
\newtheorem{lemma}[theorem]{Lemma}
\newtheorem{corollary}[theorem]{Corollary}
\theoremstyle{definition}

\newtheorem{assumption}[theorem]{Assumption}
\theoremstyle{remark}
\newtheorem{remark}[theorem]{Remark}

\usepackage{tikz}
\usetikzlibrary{arrows.meta}
\definecolor{misscol}{RGB}{176,42,42}

\title{Flow Matching with Missing Data}

\author[]{Fairoz Nower Khan}
\author[]{Nabuat Zaman Nahim}
\author[]{Peizhong~Ju}
\affil{Department of Computer Science, University of Kentucky\\ 
\mbox{Lexington, Kentucky 40506, USA}}
\affil[ ]{fairoz.khan@uky.edu, nabuat.nahim@uky.edu, peizhong.ju@uky.edu}
\date{}

\begin{document}

\maketitle

\begin{abstract}
Flow matching assumes fully observed training data, which many real-world applications rarely provide. We study Missing-Data Flow Matching, which treats the missing coordinates of training samples as latent variables and averages the flow matching loss over the values they could take. We first prove that the oracle correction is exact: completing from the true pattern-conditioned posterior recovers the complete-data flow-matching objective under any missingness mechanism. Under missing completely at random (MCAR) and missing at random (MAR), this posterior equals the ordinary complete-data conditional, while under missing
not at random (MNAR) it generally depends on the missingness mechanism and may require additional structure to identify. Our finite-sample analysis then answers design questions that the algorithm leaves open. Missingness transfers estimator variance rather than adding it, and for the fixed-parameter loss estimator one completion per example matches complete-data variance and is optimal under a fixed evaluation budget. A learned completion model contributes population completion bias, which we bound by its expected conditional Wasserstein distance to the true completion law. Experiments numerically validate the theoretical predictions, show that deterministic rather than frozen imputation reduces generated conditional variability, and place our method alongside strong classical and deep imputation baselines on real tabular data.

\end{abstract}

\section{Introduction}

Flow matching has become an effective approach to generative modeling because
it learns a vector field through a simple regression objective
\citep{lipman2023flow,tong2024improving,liu2023flow,albergo2023building}.
Standard flow matching, however, assumes that every training endpoint is fully
observed. Real data are often incomplete. Medical records omit unperformed
tests, sensors lose measurements, and surveys contain unanswered questions.
Because the flow matching loss requires a complete endpoint, missingness
prevents standard training from directly constructing the interpolation path
and its regression target.

A common response is to fill in each missing value before training
\citep{rubin1976inference,little2019statistical}. For generative learning,
however, deterministic imputation can distort the training distribution.
Replacing a range of plausible missing values by a single estimate removes
conditional variability from the endpoints. This matters because a generative
model is intended to learn a distribution, not only a point prediction.

We instead treat the missing coordinates as latent variables and complete them
stochastically before applying the ordinary flow matching loss. We call this
construction \emph{Missing-Data Flow Matching} (MDFM). Our main goal is not to
argue that this construction is a new flow matching architecture. Instead, we
characterize when it is statistically equivalent to complete-data flow
matching, what missingness changes at finite sample size, and how errors in the
completion model propagate to the learned objective and generated distribution. Figure~\ref{fig:overview} summarizes the MDFM pipeline and the main theoretical results.

\begin{figure}[t]
\centering
\begin{tikzpicture}[
  >=Latex,
  font=\footnotesize,
  stage/.style={
    draw=black!55,
    rounded corners=2pt,
    fill=black!4,
    inner sep=3pt,
    align=center,
    text width=2.45cm,
    font=\scriptsize
  },
  note/.style={
    anchor=north west,
    align=left,
    text width=4.9cm,
    font=\tiny,
    inner sep=0pt
  },
  arr/.style={
    -{Latex[length=1.8mm]},
    thick,
    black!60
  }
]

\node[stage] (obs) at (0,0)
  {incomplete observation\\[1pt]
   $O=(\Lambda,X_\Lambda)$};

\node[stage] (end) at (0,-2.25)
  {completed endpoint\\[1pt]
   $\widetilde X_1=(X_\Lambda,\widetilde X_{-\Lambda})$};

\node[stage] (loss) at (0,-4.35)
  {training objective\\[1pt]
   $\mathcal L_\phi(\theta)$};

\node[stage] (fld) at (0,-5.95)
  {learned field\\[1pt]
   $v_{\theta_\phi}$};

\node[stage] (gen) at (0,-7.45)
  {generated law\\[1pt]
   $p_1^{\theta_\phi}$};

\draw[arr] (obs)  -- (end);
\draw[arr] (end)  -- (loss);
\draw[arr] (loss) -- (fld);
\draw[arr] (fld)  -- (gen);

\node[note] at (1.55,0.42)
  {\textbf{completion:} $r^\star_\Lambda$ oracle, $q_\phi$ learned\\[2pt]
   Thm.~\ref{thm:posterior-equivalence}: oracle completion gives
   $\mathcal L_{\mathrm{MDFM}}=\mathcal L_{\mathrm{FM}}$,
   \emph{any} mechanism.\\[2pt]
   Cor.~\ref{cor:mcar-mar}: MCAR and MAR give
   $r^\star_\lambda=p^\star(\cdot\mid x_\lambda)$;
   MNAR generally needs additional identification assumptions.};

\node[note] at (1.55,-2.48)
  {\textbf{ordinary loss, unchanged}\\[2pt]
   Thm.~\ref{thm:variance-decomposition}: missingness \emph{transfers}
   variance,
   $\tau_\theta^2+\sigma_{\theta,\mathrm{miss}}^2
    =\operatorname{Var}g_\theta$.\\[1pt]
   Cor.~\ref{cor:variance-consequences}:
   $K=1$ matches complete-data variance.};

\node[note] at (1.55,-4.62)
  {Lem.~\ref{lem:completion-bias}:
   $|\mathcal L_\phi\!-\!\mathcal L_{\mathrm{FM}}|
    \le L_g\,{\color{misscol}\varepsilon_\phi}$, with\\[1pt]
   ${\color{misscol}\varepsilon_\phi}
    =\mathbb E_O W_1(q_\phi,r^\star_\Lambda)$.};

\node[note] at (1.55,-6.18)
  {Thm.~\ref{thm:completion-generation}:
   $W_2(p_1^{\theta_\phi},p^\star)\le$\\[2pt]
   \hspace*{0.6em}$C_{\mathrm{flow}}\sqrt{
   \mathcal A_\Theta
   +2\overline L_g\,{\color{misscol}\varepsilon_\phi}
   +\varepsilon_{\mathrm{opt}}
   }$.};

\end{tikzpicture}

\caption{\textbf{Overview of Missing-Data Flow Matching.}
An incomplete observation $O=(\Lambda, X_\Lambda)$ is completed with either the oracle law
$r_\Lambda^\star$ or a learned law $q_\phi$ to form a full endpoint
$\widetilde X_1=(X_\Lambda,\widetilde X_{-\Lambda})$.
MDFM then applies the ordinary flow-matching loss to this completed endpoint.
The main results show four facts. Oracle completion exactly recovers the complete-data objective.
Missingness transfers finite-sample variance rather than adding a new variance term.
Completion error is controlled by the discrepancy between $q_\phi$ and $r_\Lambda^\star$.
This error also propagates to the learned field and the generated distribution.}
\label{fig:overview}
\end{figure}

At the population level, oracle completion exactly recovers the complete-data flow-matching objective. This identity holds for arbitrary missingness mechanisms. The role of the
missingness mechanism is instead to determine whether the required completion
law can be identified and learned. Under missing completely at random (MCAR)
and missing at random (MAR), it reduces to the ordinary complete-data
conditional distribution. Under missing not at random (MNAR), it generally
depends on the missingness mechanism and may require additional assumptions for
identification. We further show that under MAR, naive self-supervised
completion learning can still be selection biased, and we characterize an
inverse-probability correction for this effect.

Our finite-sample analysis shows that missingness does not simply add variance.
With oracle completion, it redistributes variability between observations and
repeated completions. For the fixed-parameter loss estimator, one completion
per observation matches the variance of the corresponding complete-data
estimator. Additional completions reduce conditional Monte Carlo variance when
the number of observations is fixed, while under a fixed loss-evaluation budget
the estimator variance is minimized by using more observations rather than more
completions. For learned completions, we bound the resulting population bias by
the conditional Wasserstein discrepancy from the oracle completion law and
propagate this error to the velocity field and, under stability assumptions, to
the generated distribution.

Our setting is related to generative modeling from incomplete data
\citep{mattei2019miwae,ivanov2019variational,li2019misgan} and missing-data
score-based methods \citep{ouyang2023missdiff,givens2025score}. Other related works also include flow-based
imputation \citep{simkus2025cfmi} and tabular generation
\citep{jolicoeur2024generating}. These works either modify the
generative objective or use a generative model to perform imputation. Our focus
is different. We keep the ordinary flow-matching loss unchanged and analyze
when stochastic completion preserves its complete-data target and how
completion uncertainty propagates through training.

\noindent\textbf{Contributions.}
Our main contributions are:
\begin{itemize}
    \item We characterize the population problem of flow matching with
    incomplete data. Posterior completion exactly recovers complete-data flow
    matching under arbitrary missingness, and we characterize how MCAR, MAR,
    and MNAR affect identification and completion learning.

    \item We derive finite-sample variance transfer results and characterize the
    role of the number of completions. For learned completions, we bound
    completion induced bias and connect it to velocity field and generation
    error.

    \item We empirically test these predictions, including oracle equivalence,
    variance transfer, MAR self-supervision bias and its correction, deterministic
    versus stochastic completion, and sensitivity to completion quality.
\end{itemize}

\section{Related Work}
\label{relatedwork}

\noindent\textbf{Generative modeling with incomplete data.}
Several generative approaches learn from incomplete or corrupted observations,
including importance-weighted and arbitrary conditioning latent variable models,
adversarial methods, and conditional normalizing flows
\citep{mattei2019miwae,ivanov2019variational,li2019misgan,li2020acflow}.
AmbientGAN and Ambient Diffusion instead learn clean distributions from
corrupted measurements \citep{bora2018ambientgan,daras2023ambient}.
Closer to our setting, MissDiff and AugMask adapt diffusion training to
incomplete tabular observations through masked objectives and stochastic
augmentation \citep{ouyang2023missdiff,kim2026augmask}, while recent score matching methods
develop importance-weighted or variational objectives for missing observations
\citep{givens2025score}. ForestDiffusion performs both tabular generation and
imputation using diffusion and conditional flow matching
\citep{jolicoeur2024generating}. In contrast, we keep the ordinary
flow matching loss unchanged and analyze when stochastic completion recovers its
complete-data target.

\noindent\textbf{Imputation and stochastic completion.}
Classical multiple imputation represents uncertainty by considering multiple
plausible completions rather than a single fill
\citep{rubin1976inference,rubin1987multiple,little2019statistical}.
MICE and MissForest remain widely used tabular methods
\citep{van2011mice,stekhoven2012missforest}, while deep approaches include GAIN
and CSDI \citep{yoon2018gain,tashiro2021csdi}. CFMI uses conditional flow
matching specifically for missing-data imputation
\citep{simkus2025cfmi}. Such conditional imputers can instantiate the completion
model $q_\phi$ in our framework. Our focus is different from improving
imputation itself. We study how stochastic completion affects the ordinary
flow-matching objective, its finite-sample variance, and the error of the
resulting generative model.

\section{Preliminaries and Problem Setup}

Let $X\sim p^\star$ denote a complete data vector in $\mathbb R^d$, and let
$p_0$ be a base distribution. Standard flow matching independently samples
$X_0\sim p_0$, $X_1\sim p^\star$, and
$T\sim\mathrm{Unif}[0,1]$, and uses the linear path
\begin{equation}
\label{eq:fm-path}
X_t=(1-t)X_0+tX_1.
\end{equation}

For $x_t=(1-t)x_0+tx_1$, define the per-sample loss
\begin{equation}
\label{eq:fm-loss}
\ell_\theta(x_0,x_1,t)
=
\left\|
v_\theta(x_t,t)-(x_1-x_0)
\right\|_2^2.
\end{equation}
The complete-data population objective is
\begin{equation}
\label{eq:fm-objective}
\mathcal L_{\mathrm{FM}}(\theta)
=
\mathbb E
\left[
\ell_\theta(X_0,X_1,T)
\right].
\end{equation}

For the analysis, it is useful to average over the base point and time while
holding the endpoint fixed:
\begin{equation}
\label{eq:endpoint-loss}
g_\theta(x)
:=
\mathbb E_{X_0,T}
\left[
\ell_\theta(X_0,x,T)
\right].
\end{equation}
Then
\[
\mathcal L_{\mathrm{FM}}(\theta)
=
\mathbb E_{X\sim p^\star}
\left[g_\theta(X)\right].
\]
Unless stated otherwise, estimator-level results are evaluated at a fixed
parameter $\theta$. Appendix~\ref{app:notation} gives the full notation and
sources of randomness.

\paragraph{Incomplete observations.}
Let $\Lambda\subseteq[d]$ denote the random set of observed coordinates and
$-\Lambda=[d]\setminus\Lambda$ its complement. An incomplete observation is
\[
O=(\Lambda,X_\Lambda),
\]
while the coordinates $X_{-\Lambda}$ are unobserved. Let $P^\star$ denote the
joint law of $(X,\Lambda)$, with marginal $X\sim p^\star$.

For a supported pattern $\lambda$, define the true pattern-conditioned
completion law
\begin{equation}
\label{eq:oracle-completion-law}
r^\star_\lambda(\cdot\mid x_\lambda)
:=
P^\star
\left(
X_{-\lambda}\in\cdot
\mid
X_\lambda=x_\lambda,\Lambda=\lambda
\right).
\end{equation}
This is the oracle distribution of the missing coordinates given the observed
record and its missingness pattern.

Under missing completely at random (MCAR) and missing at random (MAR), this
posterior coincides with the ordinary complete-data conditional
\[
r^\star_\lambda(\cdot\mid x_\lambda)
=
p^\star_{X_{-\lambda}\mid X_\lambda}
(\cdot\mid x_\lambda).
\]
Under missing not at random (MNAR), the pattern can contain additional
information about the missing coordinates, so this equality need not hold and
the required posterior may not be identifiable without further assumptions.
Formal definitions of the three mechanisms are given in
Appendix~\ref{app:mechanisms}.

In practice $r^\star_\Lambda$ is unknown. We approximate it with a
mask-conditioned completion model
\[
q_\phi(\cdot\mid X_\Lambda,\Lambda).
\]
The next section studies what happens when the ordinary flow-matching endpoint
is supplied using either the oracle law $r^\star_\Lambda$ or its learned
approximation $q_\phi$.

\section{Missing-Data Flow Matching}
\label{sec:method}

The flow-matching loss itself remains unchanged. The only modification is how
the endpoint required by that loss is supplied.

\noindent\textbf{1) Stochastic completion changes only the endpoint.}
Given an incomplete observation $O=(\Lambda,X_\Lambda)$, draw
\[
X^\star_{-\Lambda}
\sim
r^\star_\Lambda(\cdot\mid X_\Lambda)
\]
and form the completed endpoint
$X^\star_1=(X_\Lambda,X^\star_{-\Lambda})$.
Using the endpoint loss in~\eqref{eq:endpoint-loss}, the oracle objective is
\begin{equation}
\label{eq:mdfm-oracle}
\mathcal L_{\mathrm{MDFM}}(\theta)
:=
\mathbb E_O
\mathbb E_{X^\star_{-\Lambda}\sim r^\star_\Lambda(\cdot\mid X_\Lambda)}
\left[
g_\theta(X_\Lambda,X^\star_{-\Lambda})
\right].
\end{equation}
Thus MDFM uses the ordinary flow-matching loss in~\eqref{eq:fm-loss} after
reconstructing a stochastic endpoint.

\begin{assumption}[Population regularity]
\label{ass:reg}
For every parameter $\theta$ under consideration, $v_\theta$ is measurable and
\[
\mathbb E
\left[
\ell_\theta(X_0,X_1,T)
\right]
<\infty
\]
under the complete-data path distribution.
\end{assumption}

\noindent\textbf{2) Oracle posterior completion is exactly population-correct.}
The following result shows that stochastic completion is not an approximation
when the true posterior is available.

\begin{theorem}[Posterior-completion equivalence]
\label{thm:posterior-equivalence}
Under Assumption~\ref{ass:reg}, for every $\theta$,
\begin{equation}
\label{eq:oracle-equivalence}
\mathcal L_{\mathrm{MDFM}}(\theta)
=
\mathcal L_{\mathrm{FM}}(\theta).
\end{equation}
The equality holds under an arbitrary missingness mechanism when completions
are drawn from the true pattern-conditioned posterior.
\end{theorem}

The reason is simple. Conditional on $O$, the completed missing block has
exactly the conditional law of the unobserved part of $X$. Averaging over $O$
therefore reconstructs $X\sim p^\star$. The full proof is given in
Appendix~\ref{app:posterior-equivalence}.

\begin{corollary}[MCAR and MAR specialization]
\label{cor:mcar-mar}
Under missing completely at random (MCAR) or missing at random (MAR),
\[
r^\star_\lambda(\cdot\mid x_\lambda)
=
p^\star_{X_{-\lambda}\mid X_\lambda}
(\cdot\mid x_\lambda)
\]
for almost every supported $x_\lambda$. Hence ordinary conditional completion
recovers the complete-data objective.
\end{corollary}

Under missing not at random (MNAR), Theorem~\ref{thm:posterior-equivalence}
still holds, but the required posterior generally depends on the missingness
mechanism and may not be identifiable from observed data alone. Formal
mechanism-specific results are given in Appendix~\ref{app:mechanisms}.

\begin{corollary}[Gradient agreement]
\label{cor:gradient-agreement}
If $\mathcal L_{\mathrm{MDFM}}$ and $\mathcal L_{\mathrm{FM}}$ are
differentiable at $\theta$, then
\[
\nabla_\theta\mathcal L_{\mathrm{MDFM}}(\theta)
=
\nabla_\theta\mathcal L_{\mathrm{FM}}(\theta).
\]
When differentiation and expectation can be interchanged, their expected
stochastic gradients also agree.
\end{corollary}

\begin{remark}[Deterministic imputation]
\label{rem:point-imputation}
A deterministic fill replaces the conditional completion distribution by a
point mass and therefore generally changes the population objective. Figure~\ref{fig:point-imputation} illustrates the construction, and the exact
objective gap and the special cases in which it vanishes are given in
Appendix~\ref{app:pointimp}.
\end{remark}

\noindent\textbf{3) Practical MDFM replaces the oracle by a learned completion law.}
In practice $r^\star_\Lambda$ is unknown. Let
\[
q_\phi(\cdot\mid O)
:=
q_\phi(\cdot\mid X_\Lambda,\Lambda)
\]
be a mask-conditioned completion model, and draw
$\widetilde X_{-\Lambda}\sim q_\phi(\cdot\mid O)$. The resulting population
objective is
\begin{equation}
\label{eq:learned-objective}
\mathcal L_\phi(\theta)
:=
\mathbb E_O
\mathbb E_{\widetilde X_{-\Lambda}\sim q_\phi(\cdot\mid O)}
\left[
g_\theta(X_\Lambda,\widetilde X_{-\Lambda})
\right].
\end{equation}
Training draws one or more stochastic completions for each incomplete
observation and averages the ordinary flow-matching loss. The full procedure is
given in Appendix~\ref{app:algorithm}. Section~\ref{sec:analysis} quantifies
how the discrepancy between $q_\phi$ and the oracle law affects the estimator
and the learned flow.

\noindent\textbf{4) The missingness mechanism determines both the target and how it can be learned.}
The oracle identity in Theorem~\ref{thm:posterior-equivalence} is
mechanism-independent, but representing and estimating the required posterior
is not.

\begin{proposition}[Ordinary-conditional completion]
\label{prop:naive}
Let $\mathcal L_{\mathrm{cond}}(\theta)$ denote the objective obtained by
completing from the ordinary conditional
$p^\star_{X_{-\Lambda}\mid X_\Lambda}$ rather than the
pattern-conditioned posterior. Then
\begin{equation}
\label{eq:ordinary-conditional-weight}
\begin{aligned}
\mathcal L_{\mathrm{cond}}(\theta)
&=
\mathbb E_{X\sim p^\star}
\left[
w(X)g_\theta(X)
\right],\\
w(x)
&:=
\sum_\lambda
P^\star(\Lambda=\lambda\mid X_\lambda=x_\lambda).
\end{aligned}
\end{equation}
Moreover, $\mathbb E_{p^\star}[w(X)]=1$. Under MAR, with MCAR as a special
case, $w(X)=1$ almost surely, and therefore
$\mathcal L_{\mathrm{cond}}(\theta)=\mathcal L_{\mathrm{FM}}(\theta)$.
Under MNAR, $w$ need not equal one.
\end{proposition}

Thus under MAR the ordinary complete-data conditional is the correct
population completion law. Under MNAR, it can differ from the required
pattern-conditioned posterior. The proof and the exact MNAR objective-gap
expression are given in Appendix~\ref{app:mechanisms}.

Even under MAR, however, correctly identifying the population target does not
guarantee that a naive self-supervised procedure learns it without selection
bias.

\begin{proposition}[Self-supervision under MAR]
\label{prop:mar-ssl-tilt}
Let an observed pattern $\lambda$ be split into an artificially hidden block
$A$ and retained block $R=\lambda\setminus A$. Under MAR, the population
target learned by naive self-supervision is
\begin{equation}
\label{eq:mar-ssl-tilt}
p_{\mathrm{ssl}}^{\lambda,A}(x_A\mid x_R)
=
\frac{
p^\star_{X_A\mid X_R}(x_A\mid x_R)
\pi_\lambda(x_R,x_A)
}{
\mathbb E[
\pi_\lambda(x_R,X_A)
\mid X_R=x_R]
}.
\end{equation}
Hence naive self-supervision recovers the complete-data conditional if and only
if the pattern propensity is constant in the artificially hidden variables,
conditional on the retained variables.
\end{proposition}

When $\pi_\lambda$ is identifiable and positive, stabilized
inverse-probability weighting removes this selection tilt. The proof and the
exact weighting identity are given in
Appendix~\ref{app:mar-self-supervision}. Under MNAR, additional identifying
assumptions are generally required; further discussion is given in
Appendix~\ref{app:identifiability}.

\section{Finite-Sample and Completion-Error Analysis}
\label{sec:analysis}

Theorem~\ref{thm:posterior-equivalence} is a population result. In practice,
we have finitely many incomplete observations, finitely many completion draws,
and an estimated completion law. We now characterize these three sources of
error. Unless stated otherwise, $\theta$ is fixed.

\noindent\textbf{1) Oracle completion transfers variance rather than adding it.}
Let $O_1,\ldots,O_n$ be i.i.d. incomplete observations. For each observation
$i$, draw $K$ independent oracle completions and independent base points and
times. Write
\[
Y_\theta^{(i,k)}
=
\ell_\theta
\left(
X_0^{(i,k)},
X_1^{\star,(i,k)},
T^{(i,k)}
\right),
\]
and define
\begin{equation}
\label{eq:oracle-estimator}
\widehat{\mathcal L}_{n,K}^{\star}(\theta)
=
\frac{1}{nK}
\sum_{i=1}^{n}
\sum_{k=1}^{K}
Y_\theta^{(i,k)}.
\end{equation}

For a generic observation, let
$\mu_\theta(O)=\mathbb E[Y_\theta\mid O]$ and define
\[
\tau_\theta^2
:=
\operatorname{Var}_O\!\left(\mu_\theta(O)\right),
\]
\[
\begin{aligned}
\sigma_{\theta,\mathrm{base}}^2
&:=
\mathbb E
\left[
\operatorname{Var}
\left(
Y_\theta\mid X_1^\star
\right)
\right],\\
\sigma_{\theta,\mathrm{miss}}^2
&:=
\mathbb E_O
\left[
\operatorname{Var}
\left(
g_\theta(X_1^\star)\mid O
\right)
\right].
\end{aligned}
\]

\begin{theorem}[Variance transfer]
\label{thm:variance-decomposition}
Suppose Assumption~\ref{ass:reg} holds and the complete-data loss has finite
second moment. Then
\begin{equation}
\label{eq:variance-transfer}
\begin{aligned}
\mathbb E
\left[
\widehat{\mathcal L}_{n,K}^{\star}(\theta)
\right]
&=
\mathcal L_{\mathrm{FM}}(\theta),\\
\operatorname{Var}
\left(
\widehat{\mathcal L}_{n,K}^{\star}(\theta)
\right)
&=
\frac{1}{n}
\left[
\tau_\theta^2+
\frac{
\sigma_{\theta,\mathrm{base}}^2+
\sigma_{\theta,\mathrm{miss}}^2
}{K}
\right],
\end{aligned}
\end{equation}
and
\begin{equation}
\label{eq:variance-transfer-identity}
\tau_\theta^2+\sigma_{\theta,\mathrm{miss}}^2
=
\operatorname{Var}_{X\sim p^\star}
\left(
g_\theta(X)
\right).
\end{equation}
\end{theorem}

Equation~\eqref{eq:variance-transfer-identity} is the key point. Conditional
completion moves part of the endpoint variability from the
between-observation component $\tau_\theta^2$ into the within-observation
component $\sigma_{\theta,\mathrm{miss}}^2$. It does not create an additional
variance term. The proof is given in
Appendix~\ref{app:oracle-variance}.

\begin{corollary}[Consequences for the number of completions]
\label{cor:variance-consequences}
At $K=1$, the oracle estimator has exactly the variance of the corresponding
complete-data estimator. If instead the total number of loss evaluations
$N=nK$ is fixed, then
\begin{equation}
\label{eq:fixed-budget}
\operatorname{Var}
\left(
\widehat{\mathcal L}_{N/K,K}^{\star}
\right)
=
\frac{
K\tau_\theta^2+
\sigma_{\theta,\mathrm{base}}^2+
\sigma_{\theta,\mathrm{miss}}^2
}{N}.
\end{equation}
Hence the variance is nondecreasing in $K$, and $K=1$ is optimal whenever
$\tau_\theta^2>0$.
\end{corollary}

\begin{figure}[t]
\centering
\includegraphics[width=0.86\columnwidth]{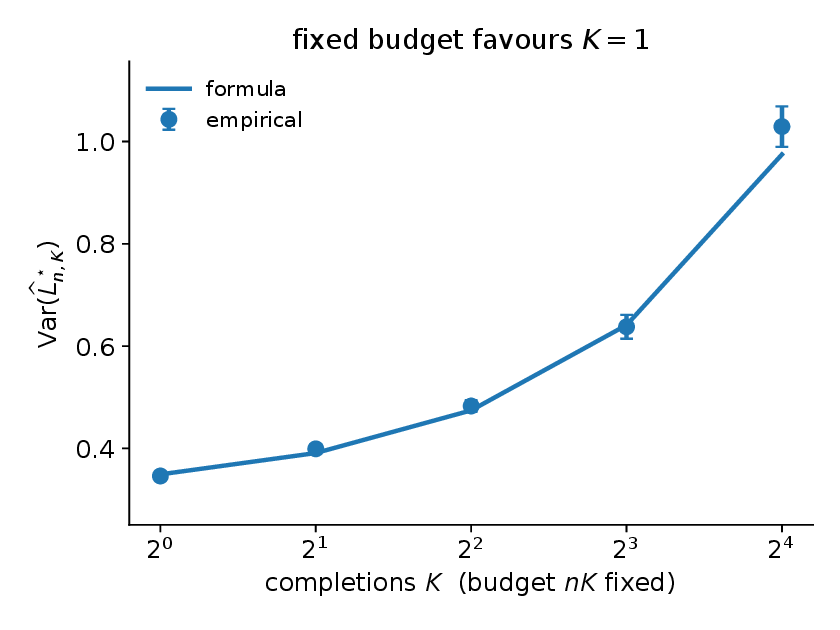}
\caption{\textbf{Fixed evaluation budget favors one completion.}
With $N=nK$ fixed, empirical estimator variance increases with $K$ and closely
tracks the prediction of Corollary~\ref{cor:variance-consequences}. Markers
show independent simulations and error bars are standard errors across five
seeds.}
\label{fig:budget}
\end{figure}

Figure~\ref{fig:budget} independently verifies this fixed-budget prediction. This fixed-budget result concerns a fixed-parameter estimator. It does not
claim that $K=1$ is optimal for an entire optimization trajectory. The same
variance-transfer structure holds for fixed-parameter stochastic gradients,
as shown in Appendix~\ref{app:gradient-variance}.

\begin{figure}[t]
\centering
\includegraphics[width=\columnwidth]{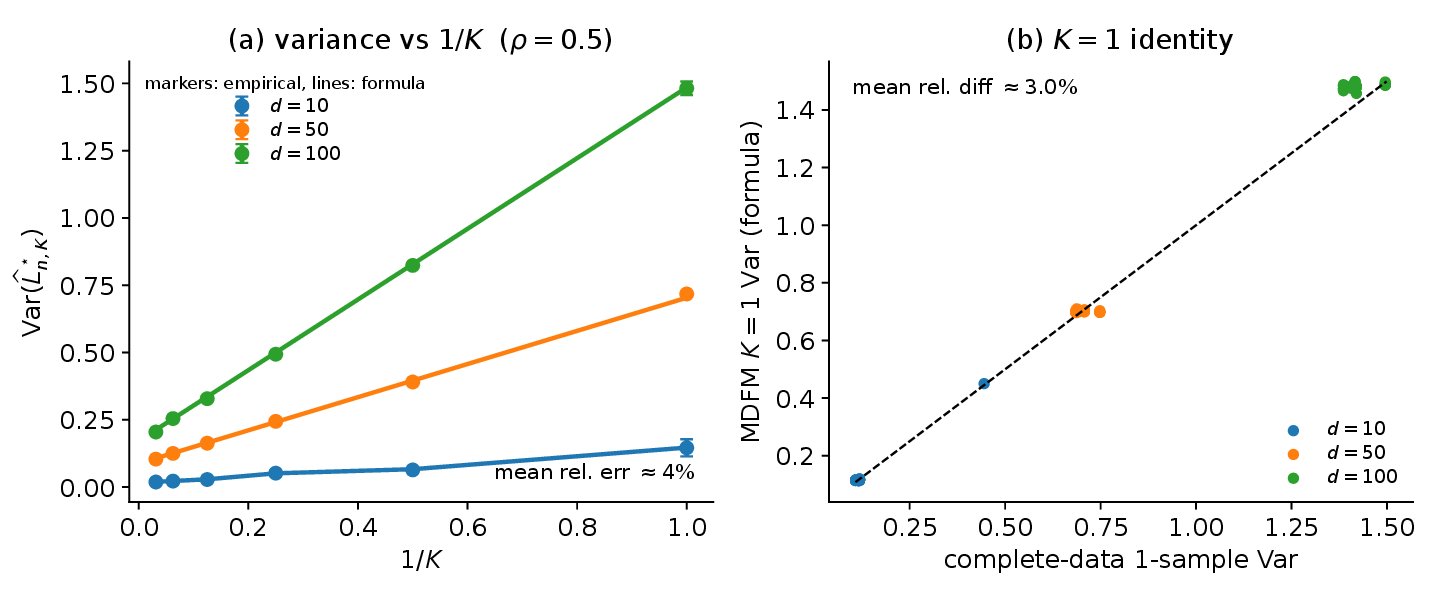}
\caption{Variance transfer predicted by
Theorem~\ref{thm:variance-decomposition}. Left, simulated estimator variance
(markers) and the theoretical prediction (lines) as $K$ varies. Right, oracle
variance at $K=1$ against the corresponding complete-data variance. The
variance components and simulated estimator variance are estimated
independently.}
\label{fig:theory}
\end{figure}

\noindent\textbf{2) The completion-specific variance is controlled by how much is missing.}
Let
\[
\rho
=
\frac{\mathbb E[|-\Lambda|]}{d}
\]
denote the expected missing fraction.

\begin{assumption}[Bounded data and endpoint regularity]
\label{ass:bdd}
There exist $R,L_g<\infty$ and a convex set $\mathcal X$ containing the
relevant endpoints such that $|X^j|\le R$ almost surely for every coordinate
$j$ and
\[
|g_\theta(x)-g_\theta(x')|
\le
L_g\|x-x'\|_2
\qquad
\text{for all }x,x'\in\mathcal X.
\]
When a learned completion model is considered, its completed endpoints are
also assumed to lie in $\mathcal X$ almost surely.
\end{assumption}

\begin{theorem}[Completion variance bound]
\label{thm:completion-variance}
Under Assumption~\ref{ass:bdd},
\begin{equation}
\label{eq:completion-variance}
\sigma_{\theta,\mathrm{miss}}^2
\le
R^2L_g^2\rho d.
\end{equation}
\end{theorem}

Thus, for fixed $n$, choosing $K=O(\rho d)$ is sufficient in the worst case
to control completion-only Monte Carlo error. It is not a necessary scaling
for the total estimator variance in
\eqref{eq:variance-transfer}. The proof and the resulting high-probability
oracle concentration bound are given in
Appendices~\ref{app:completion-variance} and
\ref{app:oracle-concentration}.

\noindent\textbf{3) Completion-model error propagates in a controlled way.}
For a learned completion law $q_\phi$, define its expected conditional
transport error by
\begin{equation}
\label{eq:completion-distance}
\varepsilon_\phi
:=
\mathbb E_O
W_1
\left(
q_\phi(\cdot\mid O),
r^\star_\Lambda(\cdot\mid X_\Lambda)
\right).
\end{equation}

\begin{lemma}[Completion-bias bound]
\label{lem:completion-bias}
Under Assumption~\ref{ass:bdd}, if the relevant conditional laws have finite
first moments, then
\begin{equation}
\label{eq:completion-bias}
\left|
\mathcal L_\phi(\theta)
-
\mathcal L_{\mathrm{FM}}(\theta)
\right|
\le
L_g\varepsilon_\phi.
\end{equation}
\end{lemma}

Thus completion quality enters as a population bias that is separate from
finite-sample Monte Carlo error. The two sources can be combined explicitly.

For observation $i$, let $Y_{\theta,\phi}^{(i,k)}$ denote the loss from the
$k$th independent completion, base-point, and time draw, with the completion
sampled from $q_\phi(\cdot\mid O_i)$, and define
\[
\begin{aligned}
\overline Y_i^\phi
&:=
\frac{1}{K}\sum_{k=1}^{K}Y_{\theta,\phi}^{(i,k)},\\
\widehat{\mathcal L}_{n,K}^{\phi}(\theta)
&:=
\frac{1}{n}\sum_{i=1}^{n}\overline Y_i^\phi,
\qquad
V_{\theta,\phi,K}
:=
\operatorname{Var}(\overline Y_i^\phi).
\end{aligned}
\]

\begin{corollary}[Finite-sample error with learned completion]
\label{cor:learned-total}
Condition on a fixed completion model $q_\phi$ independent of the evaluation
sample. Suppose
$0\le Y_{\theta,\phi}^{(i,k)}\le B_\phi$ almost surely. Then, for every
$\delta\in(0,1)$, with probability at least $1-\delta$,
\begin{equation}
\label{eq:learned-total}
\begin{aligned}
\left|
\widehat{\mathcal L}_{n,K}^{\phi}(\theta)
-
\mathcal L_{\mathrm{FM}}(\theta)
\right|
&\le
\sqrt{
\frac{
2V_{\theta,\phi,K}\log(2/\delta)
}{n}
}
\\
&\quad+
\frac{
2B_\phi\log(2/\delta)
}{3n}
+
L_g\varepsilon_\phi .
\end{aligned}
\end{equation}
\end{corollary}

The first two terms are finite-sample estimation error, while
$L_g\varepsilon_\phi$ is the population error due to imperfect completion.
The corresponding Bernstein result and proof are given in
Appendix~\ref{app:learned-concentration}; the proof of
Lemma~\ref{lem:completion-bias} is in
Appendix~\ref{app:completion-bias}.

\begin{figure}[t]
\centering
\includegraphics[width=\columnwidth]{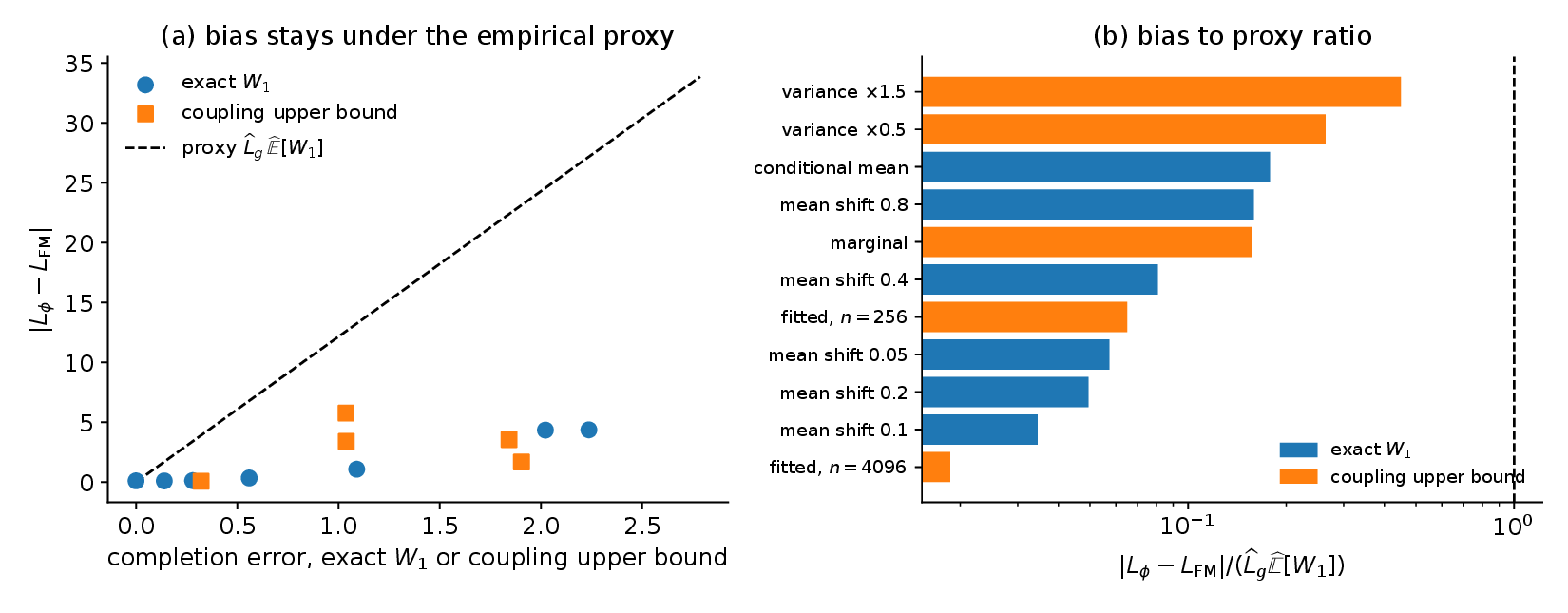}
\caption{Completion error predicts flow-matching objective error. Left,
measured $|\mathcal L_\phi-\mathcal L_{\mathrm{FM}}|$ against conditional
completion distance for twelve completion models. The dashed line is the
empirical proxy
$\widehat L_g\,\widehat{\mathbb E}[W_1]$ from
Lemma~\ref{lem:completion-bias}. Right, the ratio of measured bias to this
proxy. Since $\widehat L_g$ is estimated empirically, the line is a diagnostic
rather than a certified bound.}
\label{fig:bias}
\end{figure}

The effect also propagates beyond the objective. Let
\[
u(x,t)
:=
\mathbb E[X_1-X_0\mid X_t=x]
\]
be the population flow-matching velocity. For a model class $\Theta$, define
\begin{equation}
\label{eq:approximation-error}
\mathcal A_\Theta
:=
\inf_{\theta\in\Theta}
\int_0^1
\mathbb E
\left[
\|v_\theta(X_t,t)-u(X_t,t)\|_2^2
\right]dt.
\end{equation}
Assume the endpoint-loss Lipschitz constants are uniformly bounded by
$\overline L_g$ over $\Theta$, and let $\theta_\phi$ be an
$\varepsilon_{\mathrm{opt}}$-approximate population minimizer of
$\mathcal L_\phi$.

\begin{proposition}[Completion error controls population field error]
\label{prop:completion-field}
Under the conditions above,
\begin{equation}
\label{eq:completion-field}
\int_0^1
\mathbb E
\left[
\|v_{\theta_\phi}(X_t,t)-u(X_t,t)\|_2^2
\right]dt
\le
\mathcal A_\Theta
+
2\overline L_g\varepsilon_\phi
+
\varepsilon_{\mathrm{opt}}.
\end{equation}
\end{proposition}

Finally, assume the true and learned flows satisfy the deterministic
$W_2$ stability conditions of \citet{benton2024error}. Let
$p_1^{\theta_\phi}$ be the terminal law of the exact learned ODE and define
\[
C_{\mathrm{flow}}
=
\exp
\left\{
\int_0^1L_t\,dt
\right\},
\]
where $L_t$ is the spatial Lipschitz constant of the learned field.

\begin{theorem}[Completion-to-generation stability]
\label{thm:completion-generation}
Under the conditions above,
\begin{equation}
\label{eq:completion-generation}
W_2
\left(
p_1^{\theta_\phi},
p^\star
\right)
\le
C_{\mathrm{flow}}
\sqrt{
\mathcal A_\Theta
+
2\overline L_g\varepsilon_\phi
+
\varepsilon_{\mathrm{opt}}
}.
\end{equation}
\end{theorem}

Therefore, when approximation, optimization, and flow-regularity terms are
held fixed, the additive completion-dependent contribution scales as
$O(\sqrt{\varepsilon_\phi})$. Full derivations are given in
Appendices~\ref{app:completion-field} and
\ref{app:terminal-bound}. Appendix~\ref{app:lipschitz-class} gives an explicit
Lipschitz-controlled model class satisfying the required boundedness
conditions.

\paragraph{Scope.}
The fixed-$K$ variance conclusions are pointwise in $\theta$, and the learned
finite-sample concentration result conditions on a fixed $q_\phi$ independent
of the evaluation sample. The generation result concerns the population
minimizer and exact ODE flow. Further details are collected in
Appendix~\ref{app:scope}.

\section{Experiments}
\label{sec:experiments}

Figures~\ref{fig:theory} and~\ref{fig:bias} validate the main finite-sample
predictions. We next test the mechanism-specific MAR result and then study how
completion choices affect trained generators and real-data performance. Full
protocols and additional results are given in
Appendix~\ref{app:experiments}.

\noindent\textbf{1) MAR can bias naive self-supervised completion.}
Figure~\ref{fig:mar-ssl} tests the selection effect characterized in
Appendix~\ref{app:mar-self-supervision}. At zero MAR strength, all three
training schemes are essentially indistinguishable. As the selection strength
increases, however, naive self-supervision becomes progressively biased. Its
conditional completion error rises to about $0.19$ at $a=0.25$, while the
mask-compatible control and stabilized-IPW estimator remain near zero. The
same pattern appears downstream in the flow-matching objective: the naive
relative bias increases monotonically to about $0.10\%$, whereas both corrected
controls remain below $0.01\%$. This isolates the propensity tilt predicted by
Proposition~\ref{prop:mar-ssl-tilt} and shows that the stabilized correction of
Proposition~\ref{prop:ipw-ssl} removes it in this controlled setting.

\begin{figure}[t]
\centering
\includegraphics[width=\columnwidth]{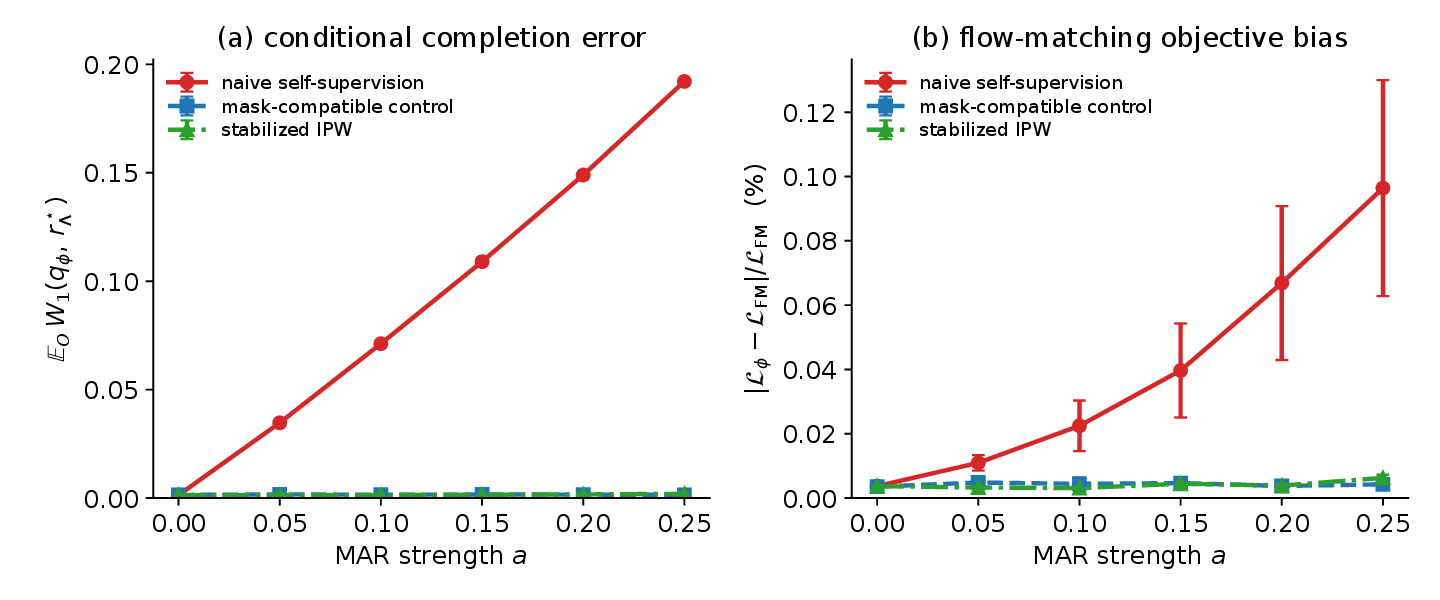}
\caption{\textbf{MAR selection bias in self-supervised completion.}
As MAR strength $a$ increases, naive self-supervision develops a growing
conditional completion error (left) and downstream flow-matching objective
bias (right). A mask-compatible control and stabilized inverse-probability
weighting remain near their sampling-error floor. Means over five seeds;
error bars show standard errors and may be smaller than the markers.}
\label{fig:mar-ssl}
\end{figure}

\noindent\textbf{2) The ordinary conditional remains valid under MAR, but not under MNAR.}
The MAR self-supervision bias above concerns how the completion law is
\emph{estimated}, not the oracle conditional itself. To separate these two
issues, Table~\ref{tab:mech} directly substitutes the ordinary complete-data
conditional into the MDFM objective under MCAR, MAR, and MNAR. As predicted by
Proposition~\ref{prop:naive}, MCAR and MAR remain at the Monte Carlo noise floor
even as the mechanism strengthens. Under MNAR, however, the objective develops
a systematic gap whose magnitude grows to nearly $5\%$. This supports the
distinction between population equivalence and identification: under MAR the
required conditional is correct but may be learned with selection bias, whereas
under MNAR the ordinary conditional itself is generally not the oracle
posterior.

\begin{table}[t]
\centering
\footnotesize
\setlength{\tabcolsep}{5pt}
\begin{tabular}{@{}lcc@{}}
\toprule
mechanism & strength &
$\frac{\mathcal L_{\mathrm{cond}}-\mathcal L_{\mathrm{FM}}}
{\mathcal L_{\mathrm{FM}}}$ \\
\midrule
MCAR & --  & $+0.14\%$ \\
MAR  & 1.0 & $+0.03\%$ \\
MAR  & 2.0 & $+0.18\%$ \\
MAR  & 3.0 & $+0.19\%$ \\
MNAR & 1.0 & $-1.91\%$ \\
MNAR & 2.0 & $-4.00\%$ \\
MNAR & 3.0 & $-4.77\%$ \\
\bottomrule
\end{tabular}
\caption{\textbf{Ordinary-conditional completion under different mechanisms.}
Relative objective gap, averaged over five seeds. MCAR and MAR stay near the
Monte Carlo floor, while MNAR develops a growing gap, consistent with
Proposition~\ref{prop:naive}.}
\label{tab:mech}
\end{table}

\begin{figure}[t]
\centering
\includegraphics[width=\columnwidth]{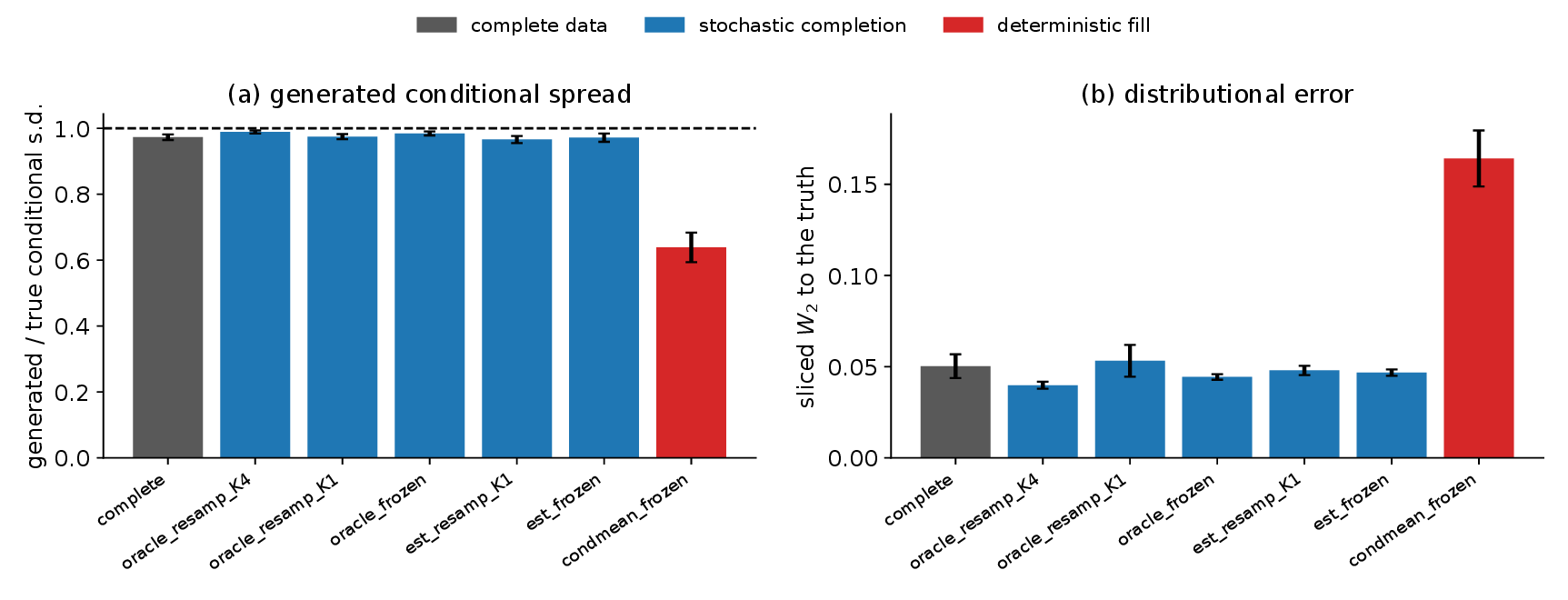}
\caption{Effect of the completion strategy on the trained generator. Left,
generated conditional standard deviation relative to the true value. Right,
sliced $W_2$ distance to the target distribution. Stochastic completion
preserves conditional spread whether completions are frozen or resampled,
whereas deterministic conditional-mean imputation does not. Mean over five
seeds.}
\label{fig:resamp}
\end{figure}

\noindent\textbf{3) Deterministic, not frozen, completion loses conditional spread.}
We train the same flow model under seven ways of supplying the missing
coordinates. Figure~\ref{fig:resamp} shows a clear separation. Conditional-mean
imputation recovers only $64\%$ of the true conditional standard deviation and
has substantially larger distributional error. Every stochastic strategy,
including frozen and resampled completions with either oracle or fitted
conditionals, recovers at least $97\%$ of the conditional spread. Thus the main
distinction is deterministic versus stochastic completion, not frozen versus
resampled completion. Full numerical results are in
Appendix~\ref{app:exp:strategies}.

\noindent\textbf{4) MDFM is competitive on real data and preserves covariance well.}
We evaluate seven UCI datasets at five MCAR missing rates using a common
flow-matching backbone. Table~\ref{tab:realdata} summarizes performance across
datasets, missing rates, and five seeds. MICE and MDFM take the top two average
ranks on all four joint-distribution metrics. MICE leads sliced $W_2$, energy
distance, and MMD, while MDFM has the best covariance error among methods
trained from incomplete data. MissForest gives the lowest imputation RMSE and
MDFM is second.

\begin{table}[t]
\centering
\footnotesize
\setlength{\tabcolsep}{2.8pt}
\begin{tabular}{@{}lccccc@{}}
\toprule
& \multicolumn{4}{c}{distributional rank $\downarrow$} & RMSE $\downarrow$ \\
\cmidrule(lr){2-5}
method & sW$_2$ & energy & MMD & cov & \\
\midrule
Complete-FM    & 1.93 & 1.69 & 1.57 & 1.72 & -- \\
Mean + FM      & 5.05 & 5.17 & 5.55 & 5.60 & 1.008 \\
MICE + FM      & \textbf{2.74} & \textbf{2.49} & \textbf{2.21} & 3.13 & 1.038 \\
MissForest + FM& 3.26 & 3.47 & 3.53 & 3.27 & \textbf{0.808} \\
GAIN + FM      & 5.18 & 5.42 & 5.17 & 4.44 & 1.115 \\
MDFM           & \underline{2.84} & \underline{2.76} & \underline{2.97}
               & \textbf{2.85} & \underline{0.895} \\
\bottomrule
\end{tabular}
\caption{Real-data results on seven UCI datasets and five missing rates, averaged
over five seeds. The first four columns are average ranks for distance to
held-out complete data. RMSE is measured on hidden training entries. Bold and
underline mark the best and second-best incomplete-data methods.}
\label{tab:realdata}
\end{table}

The differences become larger as the missing rate increases.
Figure~\ref{fig:realdata} shows that all methods are relatively close at low
missingness, but the gap widens as $\rho$ increases. MDFM remains competitive
with MICE on sliced $W_2$ and has the lowest covariance error among
incomplete-data methods at high missingness. This is the regime in which
preserving a joint distribution is most difficult.

\begin{figure}[t]
\centering
\includegraphics[width=\columnwidth]{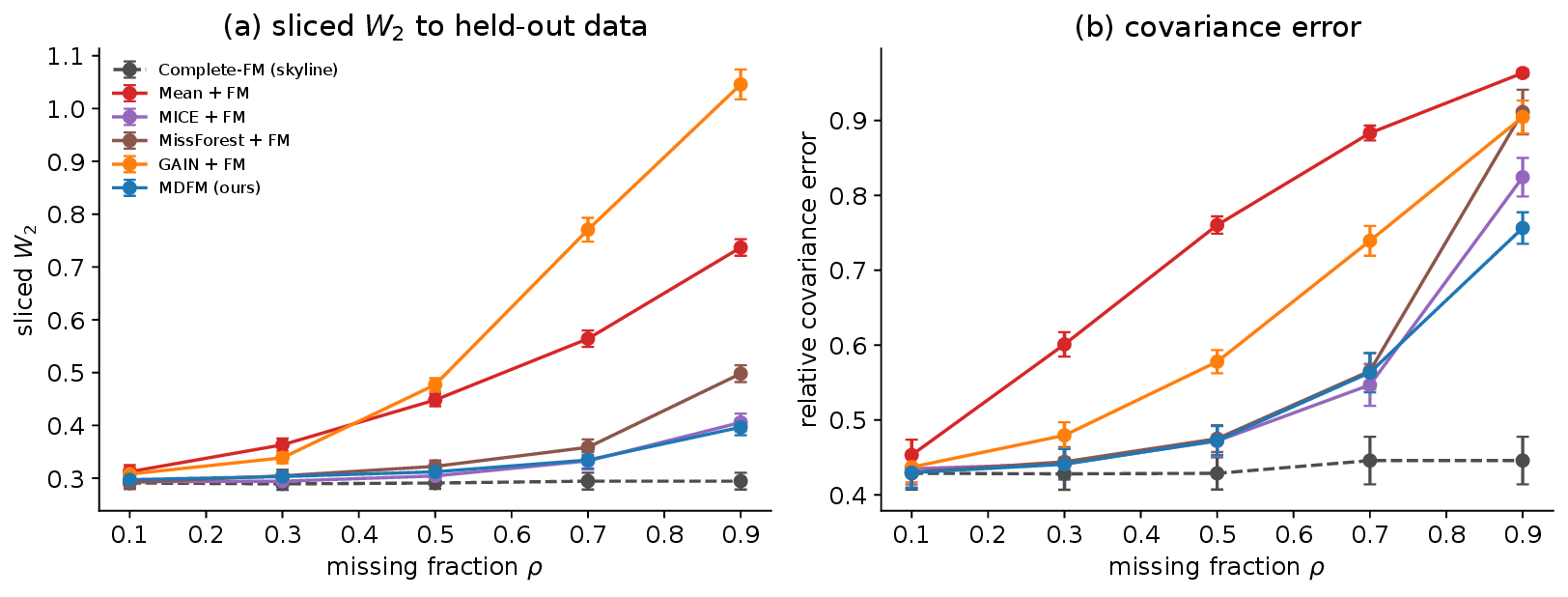}
\caption{\textbf{Real-data performance as missingness increases.}
Distributional distance to held-out complete data, averaged over seven datasets
and five seeds. Left, sliced Wasserstein distance. Right, relative covariance
error. Error bars show standard errors across seeds. Complete-FM is a skyline
that never sees missing values.}
\label{fig:realdata}
\end{figure}

These results are not intended as a state-of-the-art claim. They show that
stochastic completion remains competitive with strong classical imputers while
preserving the dependence structure needed for joint generation. Per-dataset
results and normalized-score summaries are given in
Appendix~\ref{app:exp:normalized}.

Additional checks of oracle equivalence, the missing-fraction bound, and the
fixed-budget prediction are reported in
Appendices~\ref{app:exp:equiv}, \ref{app:exp:bound}, and
\ref{app:exp:budget}.

\section{Conclusion and Future Work}

We studied flow matching with incomplete training data through
\emph{Missing-Data Flow Matching} (MDFM), which averages the ordinary
flow-matching loss over stochastic completions. Oracle completion from the true
pattern-conditioned posterior exactly recovers the complete-data objective
under any missingness mechanism. Under MCAR and MAR this posterior reduces to
the ordinary complete-data conditional, while under MNAR additional
identification assumptions are generally required. At finite sample size,
oracle completion redistributes endpoint variability rather than adding a new
variance term. For a fixed-parameter estimator, $K=1$ matches complete-data
variance and is optimal under a fixed loss-evaluation budget. For learned
completions, we separate sampling error from completion bias, bound the latter
by conditional Wasserstein error, and propagate it to the learned field and
terminal generated distribution. The experiments support these predictions and
show the importance of preserving stochastic completion.

Several questions remain open. Our variance results are pointwise in the model
parameter, and the learned-completion concentration analysis treats $q_\phi$ as
fixed and independent of the evaluation sample. Extending the theory to joint
estimation of $q_\phi$ and $v_\theta$, full optimization trajectories, and
weaker identification conditions under MAR and MNAR are natural directions for
future work.


\small
\bibliographystyle{plainnat}
\bibliography{references}

\clearpage
\appendix


\section*{Supplementary Material for Flow Matching with Missing Data}

\section{Notation and Sources of Randomness}
\label{app:notation}

Here we define the notations used throughout the paper and then state,
for every expectation and variance appearing in the analysis, exactly which
variables are averaged over and which are held fixed.

\subsection{Notational Conventions}

We follow three conventions. First, superscripts index coordinates and
subscripts index flow time, so a coordinate $X^j$ is never confused with a flow
endpoint $X_1$. Second, every expectation and variance carries a subscript
naming the variables it is taken over, together with their laws at first use;
later occurrences abbreviate once the meaning is fixed. Third, conditional
distributions are written with explicit random-variable subscripts when first
defined, so that
\[
    p^\star_{X_{-\lambda}\mid X_\lambda}(x_{-\lambda}\mid x_\lambda)
\]
names the conditional law of $X_{-\lambda}$ given $X_\lambda$ under
$X\sim p^\star$, evaluated at the values $x_{-\lambda}$ and $x_\lambda$. In the
body we abbreviate this to $p^\star(x_{-\lambda}\mid x_\lambda)$ wherever the
conditioning is unambiguous.

\subsection{Data and Missingness Conventions}

Let
\[
X=(X^1,\ldots,X^d)\in\mathbb R^d,
\qquad
[d]=\{1,\ldots,d\}.
\]
For an observed coordinate set $\Lambda\subseteq[d]$, we write
\[
-\Lambda=[d]\setminus\Lambda
\]
for the missing coordinates. The subvectors $X_\Lambda$ and
$X_{-\Lambda}$ retain their associated index sets, and
$(X_\Lambda,X_{-\Lambda})$ denotes the full vector obtained by placing the
two subvectors back in their original coordinate positions.

Equivalently, missingness may be represented by a mask
\[
M=(M^1,\ldots,M^d)\in\{0,1\}^d,
\]
where $M^j=1$ means coordinate $j$ is observed and
\[
\Lambda=\{j\in[d]:M^j=1\}.
\]
We define
\[
m(\Lambda)=|-\Lambda|,
\qquad
\rho=\frac{\mathbb E[m(\Lambda)]}{d},
\]
so that $\rho$ is the expected fraction of missing coordinates. This
definition does not assume independence between the data and the missingness
pattern.

\subsection{Notation}

\begin{center}
\small
\begin{tabular}[t]{@{}ll@{}}
\toprule
\multicolumn{2}{@{}l}{\textit{Data and missingness}}\\
$d$ & data dimension\\
$[d]$ & index set $\{1,\dots,d\}$\\
$X^j$ & coordinate $j$ of $X$\\
$\Lambda,\ -\Lambda$ & observed set, missing set $[d]\setminus\Lambda$\\
$X_\Lambda,\ X_{-\Lambda}$ & observed part in $\mathbb R^{|\Lambda|}$,\\
 & missing part in $\mathbb R^{d-|\Lambda|}$\\
$M$ & mask in $\{0,1\}^d$, $M\perp X$ under MCAR\\
$O$ & incomplete observation $(\Lambda,X_\Lambda)$\\
$m(\Lambda)$ & number of missing coordinates $|-\Lambda|$\\
$\rho$ & average missing fraction $\mathbb E[m(\Lambda)]/d$\\
$n,\ K$ & examples, completions per example\\
$i,\ k$ & index over examples, completions\\
\midrule
\multicolumn{2}{@{}l}{\textit{Distributions}}\\
$p^\star$ & true data distribution of $X$\\
$p^\star_{X_\lambda}$ & marginal of $p^\star$ on $\lambda$\\
$p^\star_{X_{-\lambda}\mid X_\lambda}$ & ordinary full-data conditional\\
$P^\star$ & joint law of $(X,\Lambda)$\\
$r^\star_{X_{-\lambda}\mid X_\lambda,\Lambda}$ & true posterior of $X_{-\lambda}$\\
 & given $x_\lambda$ and $\Lambda=\lambda$\\
$p_0$ & base distribution of $X_0$\\
$q_\phi$ & learned completion model, params $\phi$\\
$r^\star_\lambda(\cdot\mid X_\lambda)$
& true pattern-conditioned posterior completion \\
\midrule
\multicolumn{2}{@{}l}{\textit{Flow matching}}\\
$t,\ T$ & time value, random time $\sim\mathrm{Unif}[0,1]$\\
$X_0,\ X_1$ & base sample, data endpoint\\
$X_t$ & path $(1-t)X_0+tX_1$\\
$v_\theta$ & model vector field, params $\theta$\\
$\ell_\theta$ & per-sample flow matching loss\\
$g_\theta(x_1)$ & loss over $X_0,T$ at fixed endpoint $x_1$\\
\bottomrule
\end{tabular}
\hfill
\begin{tabular}[t]{@{}ll@{}}
\toprule
\multicolumn{2}{@{}l}{\textit{Objectives and estimators}}\\
$\mathcal L_{\mathrm{FM}}$ & complete-data objective\\
$\mathcal L_{\mathrm{MDFM}}$ & oracle missing-data objective\\
$\mathcal L_\phi$ & objective under learned $q_\phi$\\
$\mathcal L_{\mathrm{naive}}$ & objective under the data conditional\\
$\widehat{\mathcal L}^\star_{n,K}$ & oracle estimator, completions from
$r^\star_\Lambda(\cdot\mid X_\Lambda)$\\
$\widehat{\mathcal L}^\phi_{n,K}$ & learned estimator, completions from $q_\phi$\\
$Y_\theta$ & random per-completion loss\\
$\mu_\theta(O)$ & per-observation mean $\mathbb E[Y_\theta\mid O]$\\
$\overline Y_i$ & per-observation average over $K$\\
$\widehat G_{n,K}$ & endpoint-loss average\\
$\bar\mu$ & ideal per-observation average\\
\midrule
\multicolumn{2}{@{}l}{\textit{Constants}}\\
$\tau_\theta^2$ & between-observation variance\\
$\sigma_{\theta,\mathrm{base}}^2$ & base Monte Carlo variance\\
$\sigma_{\theta,\mathrm{miss}}^2$ & conditional completion variance\\
$R$ & coordinate bound, $|X^j|\le R$\\
$L_g$ & Lipschitz constant of $g_\theta$\\
$B,\ B_\phi$ & loss bounds, oracle and under $q_\phi$\\
$\delta$ & failure probability\\
$W_1$ & Wasserstein-1 distance\\
$\tau_{\theta,\phi}^2$ & between-observation variance under $q_\phi$\\
$\sigma_{\theta,\phi,\mathrm{base}}^2$ & base variance under $q_\phi$\\
$\sigma_{\theta,\phi,\mathrm{comp}}^2$ & completion variance under $q_\phi$\\
$V_{\rho,K},\ V_{\theta,\phi,K}$ & per-observation variance bounds\\
$\pi_\lambda(x)$ & missingness mechanism\\
$w(x)$ & MNAR reweighting factor\\
\bottomrule
\end{tabular}
\end{center}

\subsection{What Each Expectation Averages Over}

Every quantity in the analysis is defined at a fixed parameter $\theta$, so
$\theta$ is never a source of randomness. The table below records, for each
remaining quantity, which variables are averaged out and which survive.

\begin{center}
\small
\begin{tabular}{@{}lll@{}}
\toprule
Quantity & Averaged over & Remains random in\\
\midrule
$g_\theta(x_1)$ & $X_0\sim p_0$, $T\sim\mathrm{Unif}[0,1]$ & nothing\\
$\mathcal L_{\mathrm{FM}}(\theta)$ & $X_0$, $X_1\sim p^\star$, $T$ & nothing\\
$\mathcal L_{\mathrm{MDFM}}(\theta)$ & $\Lambda$, $X_\Lambda$, $X_{-\Lambda}$, $X_0$, $T$ & nothing\\
$Y_\theta$ & nothing & $O$, $X_{-\Lambda}$, $X_0$, $T$\\
$\mu_\theta(O)$ & $X_{-\Lambda}$, $X_0$, $T$ given $O$ & $O$\\
$\tau_\theta^2$ & $O$ (as a variance) & nothing\\
$\sigma_{\theta,\mathrm{base}}^2$ & $X_0,T$ inside, $X_1$ outside & nothing\\
$\sigma_{\theta,\mathrm{miss}}^2$ & $X_{-\Lambda}$ inside, $O$ outside & nothing\\
$\overline Y_i$ & $K$ inner draws & $O_i$\\
$\widehat{\mathcal L}^\star_{n,K}(\theta)$ & nothing & the whole sample\\
\bottomrule
\end{tabular}
\end{center}

Two entries deserve comment, because they are the usual source of confusion.

First, $\mu_\theta(O)$ is an expectation but not a constant. It averages out the
completion, the base point, and the time, and what survives is a function of the
incomplete observation $O$. Since $O=(\Lambda,X_\Lambda)$ is itself random, so is
$\mu_\theta(O)$, and this is exactly the randomness that
$\tau_\theta^2=\operatorname{Var}_O(\mu_\theta(O))$ measures. Writing an
expectation therefore does not always remove all randomness, and the subscript on
the outer variance is what records which randomness is left.

Second, $\sigma_{\theta,\mathrm{base}}^2$ and $\sigma_{\theta,\mathrm{miss}}^2$
are each an expectation of a variance, so they involve two nested averaging steps
over different variables. In $\sigma_{\theta,\mathrm{base}}^2$ the inner variance
is over $(X_0,T)$ at a fixed endpoint and the outer expectation is over the
endpoint. In $\sigma_{\theta,\mathrm{miss}}^2$ the inner variance is over the
completion at a fixed observation and the outer expectation is over the
observation. Interchanging the two would give a different quantity, which is why
both subscripts are stated.

Finally, the estimator variance
$\operatorname{Var}(\widehat{\mathcal L}^\star_{n,K}(\theta))$ in Theorem~\ref{thm:variance-decomposition} is taken over the full sampling procedure, namely the $n$ incomplete observations,
the $nK$ completions, and the $nK$ base points and times, with $\theta$, $n$, and
$K$ held fixed.

\section{Standard Probabilistic Tools and Their Specialization to Our Setting}
\label{app:tools}

Here we state the standard results used in our proofs, together with the
substitutions that specialize each one to our setting. Throughout, all random
variables are defined on a common probability space and are assumed integrable to
the order required.

\subsection{Tower Property and Disintegration}

\begin{lemma}[Tower property]
\label{lem:condexp}
Let $U$ be an integrable real random variable and $V$ any random element on the
same space. Then $\mathbb E\big[\mathbb E[U\mid V]\big]=\mathbb E[U]$. More
generally, if $\kappa(\cdot\mid v)$ is a regular conditional distribution of $W$
given $V=v$ and $h$ is integrable under the joint law, then
\begin{equation}
\label{eq:disint}
    \mathbb E_V\Big[\int h(V,w)\,\kappa(dw\mid V)\Big]=\mathbb E\big[h(V,W)\big].
\end{equation}
\end{lemma}

\begin{proof}
The first claim is the defining property of conditional expectation applied to
the trivial $\sigma$-algebra. For \eqref{eq:disint}, the inner integral equals
$\mathbb E[h(V,W)\mid V]$ almost surely by definition of a regular conditional
distribution, and taking expectations gives the claim by the first part.
\end{proof}

\paragraph{Where we use it.} In Theorem~\ref{thm:posterior-equivalence} we take $V=O$,
$W=X_{-\Lambda}$, $\kappa=r^\star_\Lambda(\cdot\mid X_\Lambda)$, and
$h=g_\theta$. In Corollary~\ref{cor:mcar-mar} the same substitution is made
with $V=X_\lambda$, $W=X_{-\lambda}$, and
$\kappa=p^\star_{X_{-\lambda}\mid X_\lambda}(\cdot\mid x_\lambda)$, which is
the form the completion law takes under MCAR and MAR.

\subsection{Law of Total Variance}

\begin{lemma}[Law of total variance]
\label{lem:ltv}
Let $U$ be a square-integrable real random variable and let $V$ be any random
element on the same space. Then
\begin{equation}
\label{eq:ltv}
    \operatorname{Var}(U)
    =\mathbb E\big[\operatorname{Var}(U\mid V)\big]
    +\operatorname{Var}\big(\mathbb E[U\mid V]\big).
\end{equation}
If $W$ is a further random element, the same identity holds conditionally,
\begin{equation}
\label{eq:ltvcond}
    \operatorname{Var}(U\mid W)
    =\mathbb E\big[\operatorname{Var}(U\mid V,W)\,\big|\,W\big]
    +\operatorname{Var}\big(\mathbb E[U\mid V,W]\,\big|\,W\big).
\end{equation}
\end{lemma}

\begin{proof}
Write $m(V)=\mathbb E[U\mid V]$. By definition of conditional variance,
\[
    \mathbb E\big[\operatorname{Var}(U\mid V)\big]
    =\mathbb E\big[\mathbb E[U^2\mid V]-m(V)^2\big]
    =\mathbb E[U^2]-\mathbb E[m(V)^2],
\]
using the tower property for the first term. Also
\[
    \operatorname{Var}(m(V))
    =\mathbb E[m(V)^2]-\big(\mathbb E[m(V)]\big)^2
    =\mathbb E[m(V)^2]-\big(\mathbb E[U]\big)^2 .
\]
Adding the two displays gives
$\mathbb E[U^2]-(\mathbb E[U])^2=\operatorname{Var}(U)$, which is
\eqref{eq:ltv}. Identity \eqref{eq:ltvcond} is \eqref{eq:ltv} applied under the
regular conditional distribution given $W$.
\end{proof}

\paragraph{Where we use it.} Lemma~\ref{lem:ltv} is applied three times, with the
following substitutions.

\begin{center}
\small
\begin{tabular}{@{}lll@{}}
\toprule
Use & $U$ & conditioning \\
\midrule
Theorem~\ref{thm:variance-decomposition}, transfer & $g_\theta(X_1)$ & $V=O$\\
Theorem~\ref{thm:variance-decomposition}, step 2 & $\overline Y_i$ & $V=O_i$\\
Theorem~\ref{thm:variance-decomposition}, step 3 & $Y_\theta^{(i,1)}$ & $V=X^{\star,(i,1)}_{-\Lambda_i}$, $W=O_i$\\
\bottomrule
\end{tabular}
\end{center}

In the first use, $\mathbb E[U\mid V]=\mu_\theta(O)$ and
$\operatorname{Var}(U\mid V)$ is the conditional variance of $g_\theta$ over
completions, so \eqref{eq:ltv} reads
$\operatorname{Var}(g_\theta(X_1))=\sigma^2_{\theta,\mathrm{miss}}+\tau_\theta^2$.
In the third use, \eqref{eq:ltvcond} is the relevant form, since the
decomposition is carried out conditionally on the observation $O_i$.

\subsection{Popoviciu's Inequality}

\begin{lemma}[Popoviciu's inequality]
\label{lem:popoviciu}
Let $Z$ be a random variable with $a\le Z\le b$ almost surely. Then
\[
    \operatorname{Var}(Z)\le\frac{(b-a)^2}{4}.
\]
\end{lemma}

\begin{proof}
The variance is the minimum mean squared deviation, so for the midpoint
$c=(a+b)/2$,
\[
    \operatorname{Var}(Z)\le\mathbb E\big[(Z-c)^2\big].
\]
Since $Z\in[a,b]$ almost surely, $|Z-c|\le (b-a)/2$ almost surely, whence
$\mathbb E[(Z-c)^2]\le (b-a)^2/4$.
\end{proof}

The inequality is due to Popoviciu \citep{popoviciu1935equations}; sharper
variants exist but are not needed here.

\paragraph{Where we use it.} In Theorem~\ref{thm:completion-variance} we take
$Z=g_\theta(x_\Lambda,X_{-\Lambda})$ conditional on an observation $O$ with $m$
missing coordinates, with $a$ and $b$ the essential infimum and supremum of $Z$
given $O$. The Lipschitz and boundedness assumptions give $b-a\le 2RL_g\sqrt m$,
so Lemma~\ref{lem:popoviciu} yields
$\operatorname{Var}(Z\mid O)\le R^2L_g^2 m$.

\subsection{Bernstein's Inequality}

\begin{lemma}[Bernstein's inequality, bounded case]
\label{lem:bernstein}
Let $Z_1,\dots,Z_n$ be i.i.d.\ real random variables with mean $\mu$, variance at
most $v$, and $|Z_i-\mu|\le M$ almost surely. Then for every $\varepsilon>0$,
\[
    \mathbb P\Big(\Big|\tfrac1n\textstyle\sum_{i=1}^n Z_i-\mu\Big|\ge\varepsilon\Big)
    \le 2\exp\Big(-\frac{n\varepsilon^2}{2v+\tfrac{2}{3}M\varepsilon}\Big).
\]
Equivalently, for every $\delta\in(0,1)$, with probability at least $1-\delta$,
\begin{equation}
\label{eq:bernstein_inv}
    \Big|\tfrac1n\textstyle\sum_{i=1}^n Z_i-\mu\Big|
    \le\sqrt{\frac{2v\log(2/\delta)}{n}}+\frac{2M\log(2/\delta)}{3n}.
\end{equation}
\end{lemma}

\begin{proof}
The tail bound is the classical Bernstein inequality for bounded independent
variables; see for instance Theorem 2.10 and the following corollary in
\citet{boucheron2013concentration}. For the inverted form, write
$L=\log(2/\delta)$ and require the exponent to be at most $-L$, that is
\[
    n\varepsilon^2-\tfrac{2}{3}ML\varepsilon-2vL\ \ge\ 0 .
\]
The positive root of the corresponding quadratic is
\[
    \varepsilon_\star=\frac{ML}{3n}+\sqrt{\Big(\frac{ML}{3n}\Big)^2+\frac{2vL}{n}},
\]
and any $\varepsilon\ge\varepsilon_\star$ satisfies the inequality. Applying
$\sqrt{x+y}\le\sqrt x+\sqrt y$ to the square root gives
$\varepsilon_\star\le \frac{2ML}{3n}+\sqrt{2vL/n}$, which is
\eqref{eq:bernstein_inv}.
\end{proof}

\paragraph{Where we use it.} In Theorem~\ref{thm:posterior-equivalence-bernstein} we take
$Z_i=\overline Y_i$, the per-observation average, with $\mu=\mathcal
L_{\mathrm{FM}}(\theta)$, $v=V_{\rho,K}$, and $M=B$, the latter because
$\overline Y_i\in[0,B]$ implies $|\overline Y_i-\mu|\le B$. In
Theorem~\ref{thm:learned-bernstein} the same substitution is made with
$\mu=\mathcal L_\phi(\theta)$, $v=V_{\theta,\phi,K}$, and $M=B_\phi$.

\subsection{Kantorovich--Rubinstein Duality}

\begin{lemma}[Kantorovich--Rubinstein duality]
\label{lem:kr}
Let $\nu_1,\nu_2$ be probability measures on $\mathbb R^m$ with finite first
moments and let $h:\mathbb R^m\to\mathbb R$ be $L$-Lipschitz. Then
\[
    \Big|\int h\,d\nu_1-\int h\,d\nu_2\Big|\le L\,W_1(\nu_1,\nu_2),
\]
where $W_1$ is the Wasserstein-1 distance.
\end{lemma}

\begin{proof}
Duality states $W_1(\nu_1,\nu_2)=\sup\{\int f d\nu_1-\int f d\nu_2\}$ over
$1$-Lipschitz $f$, a classical result for which we refer to
\citet{villani2009optimal}. Applying this to $f=h/L$, which is $1$-Lipschitz, gives
$|\int h d\nu_1-\int h d\nu_2|\le L\,W_1(\nu_1,\nu_2)$.
\end{proof}

\paragraph{Where we use it.} In Lemma~\ref{lem:completion-bias} we take $h$ to be $g_\theta$ as a function of the missing coordinates with the observed part held fixed, extended off $\mathcal X_O$ by McShane, $L=L_g$, $\nu_1=q_\phi(\cdot\mid O)$, and
$\nu_2=r^\star_\Lambda(\cdot\mid X_\Lambda)$.

\subsection{Sufficient Conditions for the Regularity Assumption}
\label{app:reg}

The nontrivial integrability requirement in Assumption~\ref{ass:reg} follows from a standard linear-growth condition satisfied by common parameterizations. Suppose $v_\theta$ is continuous in $(x,t)$, which makes it Borel
measurable, and has at most linear growth in its spatial argument,
\[
    \|v_\theta(x,t)\|\le C_\theta\big(1+\|x\|\big)
    \qquad\text{for all } x\in\mathbb R^d,\ t\in[0,1].
\]
Because $X_t=(1-t)X_0+tX_1$ is a convex combination, finite second moments of
$p_0$ and $p^\star$ give
$\mathbb E\|X_t\|^2\le 2(\mathbb E\|X_0\|^2+\mathbb E\|X_1\|^2)<\infty$, hence
$\mathbb E\|v_\theta(X_t,T)\|^2\le 2C_\theta^2(1+\mathbb E\|X_t\|^2)<\infty$ and
$\mathcal L_{\mathrm{FM}}(\theta)<\infty$.

The variance results of Section~\ref{sec:analysis} additionally require
$\mathbb E[\ell_\theta(X_0,X_1,T)^2]<\infty$. This follows from the same linear
growth condition together with finite fourth moments of $p_0$ and $p^\star$,
since linear growth gives
$\ell_\theta\le 2\|v_\theta(X_t,T)\|^2+2\|X_1-X_0\|^2\le C_\theta'(1+\|X_0\|^2+\|X_1\|^2)$
for a constant $C_\theta'$, and squaring the right side leaves only fourth
moments. Under Assumption~\ref{ass:bdd} with a compactly supported $p_0$ it holds
automatically.

Linear growth holds for any feedforward network with Lipschitz activations, such
as ReLU, GELU, SiLU, or tanh, since such a network is globally Lipschitz in $x$
with constant bounded by the product of its layer operator norms, giving
$\|v_\theta(x,t)\|\le\|v_\theta(0,t)\|+L_\theta\|x\|$. The straight line
interpolant is what keeps this clean, since $X_t$ has finite second moment at
every $t$ and the target $X_1-X_0$ carries no time dependent normalization that
could diverge as $t\to1$. The same holds for rectified flow
\citep{liu2023flow}, which regresses the same velocity target along the same
interpolant, so it inherits the condition unchanged.

Two points keep the assumption honest. It is imposed at a fixed $\theta$ rather
than uniformly over the parameter space, which is all our fixed $\theta$ results
require. And since $\ell_\theta\ge0$, the marginalization identity in the proof of
Theorem~\ref{thm:posterior-equivalence} follows from Tonelli's theorem under measurability
alone, so square integrability serves only to guarantee that both sides are
finite rather than trivially infinite.

\section{Full Proofs of the Results in Sections 4 and 5}
\label{app:full-proofs}

\subsection{Posterior-Completion Equivalence}
\label{app:posterior-equivalence}

\begin{proof}[Proof of Theorem~\ref{thm:posterior-equivalence}]
Fix $\theta$ and write
\[
O=(\Lambda,X_\Lambda).
\]
By definition,
\[
r^\star_\Lambda(\cdot\mid X_\Lambda)
=
P^\star
\left(
X_{-\Lambda}\in\cdot
\mid
X_\Lambda,\Lambda
\right)
\]
is a regular conditional distribution of $X_{-\Lambda}$ given
$O$. Therefore, for $P^\star_O$-almost every $O$,
\[
\mathbb E_{
X_{-\Lambda}\sim
r^\star_\Lambda(\cdot\mid X_\Lambda)
}
\left[
g_\theta(X_\Lambda,X_{-\Lambda})
\mid O
\right]
=
\mathbb E
\left[
g_\theta(X)\mid O
\right].
\]
Taking expectation over $O$ and applying the tower property,
\begin{align*}
L_{\mathrm{MDFM}}(\theta)
&=
\mathbb E_O
\mathbb E_{
X_{-\Lambda}\sim
r^\star_\Lambda(\cdot\mid X_\Lambda)
}
\left[
g_\theta(X_\Lambda,X_{-\Lambda})
\right]
\\
&=
\mathbb E_O
\left[
\mathbb E[g_\theta(X)\mid O]
\right]
\\
&=
\mathbb E_{X\sim p^\star}
[g_\theta(X)]
\\
&=
L_{\mathrm{FM}}(\theta).
\end{align*}
No independence assumption between $X$ and $\Lambda$ is used.
\end{proof}

\begin{proof}[Proof of Corollary~\ref{cor:mcar-mar}]
Fix a pattern $\lambda$ with
$\Pr(\Lambda=\lambda)>0$.

Under MAR,
\[
\Pr(\Lambda=\lambda\mid X=x)
=
\Pr(\Lambda=\lambda\mid X_\lambda=x_\lambda)
=:\pi_\lambda(x_\lambda).
\]
For any measurable set $A$ in the space of $X_{-\lambda}$,
Bayes' rule gives, for almost every $x_\lambda$,
\begin{align*}
&
P^\star
\left(
X_{-\lambda}\in A
\mid
X_\lambda=x_\lambda,\Lambda=\lambda
\right)
\\
&\quad=
\frac{
\displaystyle
\int_A
\Pr(\Lambda=\lambda\mid
X_\lambda=x_\lambda,X_{-\lambda}=u)
\,p^\star(du\mid x_\lambda)
}{
\Pr(\Lambda=\lambda\mid X_\lambda=x_\lambda)
}.
\end{align*}
Under MAR the factor in the numerator equals
$\pi_\lambda(x_\lambda)$ and is independent of $u$, so it
cancels with the denominator. Hence
\[
r^\star_\lambda(\cdot\mid x_\lambda)
=
p^\star(\cdot\mid x_\lambda)
\qquad\text{a.e.}
\]
MCAR is the special case in which
$\pi_\lambda(x_\lambda)=\Pr(\Lambda=\lambda)$ is constant.

The conclusion
\[
L_{\mathrm{MDFM}}(\theta)
=
L_{\mathrm{FM}}(\theta)
\]
then follows from
Theorem~\ref{thm:posterior-equivalence}.
\end{proof}

\begin{proof}[Proof of Corollary~\ref{cor:gradient-agreement}]
Theorem~\ref{thm:posterior-equivalence} gives
\[
L_{\mathrm{MDFM}}(\theta)
=
L_{\mathrm{FM}}(\theta)
\]
for every $\theta$ in their common domain. Therefore, at every
$\theta$ at which both sides are differentiable,
\[
\nabla_\theta L_{\mathrm{MDFM}}(\theta)
=
\nabla_\theta L_{\mathrm{FM}}(\theta).
\]

If additionally differentiation may be exchanged with expectation,
then
\[
\nabla_\theta L_{\mathrm{MDFM}}(\theta)
=
\mathbb E[
\nabla_\theta\ell_\theta]
\]
under the oracle completion construction, and likewise
\[
\nabla_\theta L_{\mathrm{FM}}(\theta)
=
\mathbb E[
\nabla_\theta\ell_\theta]
\]
under complete data. Hence their expected stochastic gradients also
coincide.
\end{proof}

\subsection{Self-Supervised Completion under MAR}
\label{app:mar-self-supervision}

Suppose an observation has original pattern $\Lambda=\lambda$. A
self-supervised completion procedure artificially hides a subset
$A\subseteq\lambda$ and retains
\[
R=\lambda\setminus A.
\]
Let $\nu_\lambda$ denote the artificial masking distribution, and assume that
the artificial mask is conditionally independent of the data values given
$\Lambda=\lambda$.

Under MAR, define
\[
\pi_\lambda(x_\lambda)
:=
P^\star(\Lambda=\lambda\mid X_\lambda=x_\lambda).
\]

\begin{proof}[Proof of Proposition~\ref{prop:mar-ssl-tilt}]
Fix an original observed pattern $\lambda$ and an artificial split
\[
\lambda=A\cup R,
\qquad
A\cap R=\varnothing.
\]
Under MAR define
\[
\pi_\lambda(x_\lambda)
=
\Pr(\Lambda=\lambda\mid X_\lambda=x_\lambda).
\]
The density of $X_\lambda$ conditional on observing pattern
$\lambda$ is
\[
p(x_\lambda\mid\Lambda=\lambda)
=
\frac{
p^\star_{X_\lambda}(x_\lambda)
\pi_\lambda(x_\lambda)
}{
\Pr(\Lambda=\lambda)
}.
\]
Writing $x_\lambda=(x_R,x_A)$ and factorizing the full-data
marginal,
\[
p^\star_{X_\lambda}(x_R,x_A)
=
p^\star_{X_R}(x_R)
p^\star(x_A\mid x_R),
\]
we obtain
\[
p(x_R,x_A\mid\Lambda=\lambda)
=
\frac{
p^\star_{X_R}(x_R)
p^\star(x_A\mid x_R)
\pi_\lambda(x_R,x_A)
}{
\Pr(\Lambda=\lambda)
}.
\]
Therefore,
\begin{align*}
&p_{\mathrm{ssl}}
(x_A\mid x_R,\Lambda=\lambda)
\\
&\quad=
\frac{
p^\star(x_A\mid x_R)
\pi_\lambda(x_R,x_A)
}{
\displaystyle
\int
p^\star(z_A\mid x_R)
\pi_\lambda(x_R,z_A)\,dz_A
}
\\
&\quad=
\frac{
p^\star(x_A\mid x_R)
\pi_\lambda(x_R,x_A)
}{
\mathbb E_{
X_A\sim p^\star(\cdot\mid x_R)
}
[
\pi_\lambda(x_R,X_A)
]
}.
\end{align*}
The probability of selecting the artificial mask $A$ does not
appear because, conditional on the original pattern $\lambda$, the
artificial masking distribution is assumed independent of the data
values.

It follows immediately that the self-supervised conditional equals
$p^\star(x_A\mid x_R)$ whenever
$\pi_\lambda(x_R,x_A)$ is constant in $x_A$ conditional on $x_R$.
Conversely, whenever the propensity varies with $x_A$, the
self-supervised law is propensity tilted.
\end{proof}

\begin{proposition}[Rich self-supervision forces MCAR]
\label{prop:rich-mask-mcar}
Fix a supported pattern $\lambda$. Suppose the support of $X_\lambda$ is a
Cartesian product with positive density almost everywhere, and suppose every
single coordinate can be selected as an artificial mask:
\[
\nu_\lambda(\{j\})>0
\qquad
\text{for every }j\in\lambda.
\]
If naive self-supervision targets the correct complete-data conditional for
every single-coordinate artificial mask, then
\[
\pi_\lambda(x_\lambda)=c_\lambda
\qquad
p^\star_{X_\lambda}\text{-a.e.}
\]
for some constant $c_\lambda$. If this holds for every supported pattern, then
the missingness mechanism is MCAR.
\end{proposition}

\begin{proof}[Proof of Proposition~\ref{prop:rich-mask-mcar}]
Fix a supported pattern $\lambda$.

For any $j\in\lambda$, consider the artificial mask
\[
A=\{j\},
\qquad
R=\lambda\setminus\{j\}.
\]
By Proposition~\ref{prop:mar-ssl-tilt}, unbiased naive
self-supervision for this split requires
\[
\pi_\lambda(x_\lambda)
\]
to be independent of $x_j$ once
$x_{\lambda\setminus\{j\}}$ is fixed. Equivalently, there exists a
measurable function $f_j$ such that
\[
\pi_\lambda(x_\lambda)
=
f_j(x_{\lambda\setminus\{j\}})
\qquad
p^\star_{X_\lambda}\text{-a.e.}
\]

By assumption this holds for every $j\in\lambda$. Hence
$\pi_\lambda$ is invariant, almost everywhere, to changes in each
coordinate separately. On the assumed Cartesian-product support,
with positive density almost everywhere, repeated use of Fubini's
theorem implies that $\pi_\lambda$ must be constant almost
everywhere:
\[
\pi_\lambda(x_\lambda)=c_\lambda
\qquad
p^\star_{X_\lambda}\text{-a.e.}
\]

Under MAR,
\[
\Pr(\Lambda=\lambda\mid X=x)
=
\pi_\lambda(x_\lambda)
=
c_\lambda.
\]
Taking expectation over $X$ gives
\[
c_\lambda
=
\Pr(\Lambda=\lambda).
\]
Thus
\[
\Pr(\Lambda=\lambda\mid X)
=
\Pr(\Lambda=\lambda)
\qquad\text{a.s.}
\]
for this pattern. If the same conclusion holds for every supported
pattern $\lambda$, then the missingness mechanism is MCAR.
\end{proof}

\begin{proposition}[Stabilized inverse-probability correction]
\label{prop:ipw-ssl}
Assume MAR and
\[
\pi_\lambda(X_\lambda)>0
\qquad
p^\star_{X_\lambda}\text{-a.s.}
\]
for every supported pattern $\lambda$. Define
\[
w_\lambda(x_\lambda)
=
\frac{
P^\star(\Lambda=\lambda)
}{
\pi_\lambda(x_\lambda)
}.
\]
Then, for every integrable measurable function $f_\lambda$,
\begin{equation}
\label{eq:ipw-pattern}
\mathbb E
\left[
\mathbf 1\{\Lambda=\lambda\}
w_\lambda(X_\lambda)
f_\lambda(X_\lambda)
\right]
=
P^\star(\Lambda=\lambda)
\mathbb E_{X_\lambda\sim p^\star_{X_\lambda}}
\left[
f_\lambda(X_\lambda)
\right].
\end{equation}
Equivalently,
\[
\mathbb E
\left[
w_\lambda(X_\lambda)
f_\lambda(X_\lambda)
\mid
\Lambda=\lambda
\right]
=
\mathbb E_{X_\lambda\sim p^\star_{X_\lambda}}
\left[
f_\lambda(X_\lambda)
\right].
\]
Thus stabilized inverse-probability weighting removes the pattern-selection
tilt from the self-supervised population objective.
\end{proposition}

\begin{proof}[Proof of Proposition~\ref{prop:ipw-ssl}]
For each supported pattern $\lambda$, define
\[
\pi_\lambda(x_\lambda)
=
\Pr(\Lambda=\lambda\mid X_\lambda=x_\lambda)
\]
and the stabilized weight
\[
w_\lambda(x_\lambda)
=
\frac{
\Pr(\Lambda=\lambda)
}{
\pi_\lambda(x_\lambda)
}.
\]
Under MAR,
\[
p(x_\lambda\mid\Lambda=\lambda)
=
\frac{
p^\star_{X_\lambda}(x_\lambda)
\pi_\lambda(x_\lambda)
}{
\Pr(\Lambda=\lambda)
}.
\]
Therefore, for every integrable $f_\lambda$,
\begin{align*}
&
\mathbb E
\left[
w_\lambda(X_\lambda)
f_\lambda(X_\lambda)
\mid
\Lambda=\lambda
\right]
\\
&=
\int
\frac{\Pr(\Lambda=\lambda)}
{\pi_\lambda(x_\lambda)}
f_\lambda(x_\lambda)
\frac{
p^\star_{X_\lambda}(x_\lambda)
\pi_\lambda(x_\lambda)
}{
\Pr(\Lambda=\lambda)
}
\,dx_\lambda
\\
&=
\int
f_\lambda(x_\lambda)
p^\star_{X_\lambda}(x_\lambda)\,dx_\lambda
\\
&=
\mathbb E_{X_\lambda\sim p^\star_{X_\lambda}}
[f_\lambda(X_\lambda)].
\end{align*}
Multiplying by $\Pr(\Lambda=\lambda)$ yields
\[
\mathbb E
\left[
\mathbf 1\{\Lambda=\lambda\}
w_\lambda(X_\lambda)
f_\lambda(X_\lambda)
\right]
=
\Pr(\Lambda=\lambda)
\mathbb E_{p^\star_{X_\lambda}}
[f_\lambda(X_\lambda)].
\]

Now let
\[
f_\lambda(X_\lambda)
=
\mathbb E_{A\sim\nu_\lambda}
[
\mathcal J_\phi(X_A,X_R;A,R)
].
\]
Summing the preceding identity over $\lambda$ gives
\begin{align*}
&
\mathbb E
\left[
w_\Lambda(X_\Lambda)
\mathbb E_{A\sim\nu_\Lambda}
\mathcal J_\phi(X_A,X_R;A,R)
\right]
\\
&=
\sum_\lambda
\Pr(\Lambda=\lambda)
\mathbb E_{X_\lambda\sim p^\star_{X_\lambda}}
\mathbb E_{A\sim\nu_\lambda}
[
\mathcal J_\phi(X_A,X_R;A,R)
].
\end{align*}
Thus stabilized IPW removes the pattern-selection tilt from the
self-supervised objective.

Finally,
\[
\mathbb E[
w_\lambda(X_\lambda)\mid\Lambda=\lambda]
=1,
\]
but this identity does not imply that the weights are bounded; weak
positivity can still produce large weights.
\end{proof}

\subsection{Oracle Unbiasedness and Variance Decomposition}
\label{app:oracle-variance}

For observation $i$ and completion $k$, define
\[
Y_\theta^{(i,k)}
=
\ell_\theta
\left(
X_0^{(i,k)},
X_1^{\star,(i,k)},
T^{(i,k)}
\right),
\]
where
\[
X_{-\Lambda_i}^{\star,(i,k)}
\sim
r^\star_{\Lambda_i}
\left(
\cdot\mid X_{\Lambda_i}^{(i)}
\right),
\]
and
\[
X_1^{\star,(i,k)}
=
\left(
X_{\Lambda_i}^{(i)},
X_{-\Lambda_i}^{\star,(i,k)}
\right).
\]
Also define
\[
\overline Y_i
=
\frac1K\sum_{k=1}^K
Y_\theta^{(i,k)}
\]
and
\[
\mu_\theta(O)
=
\mathbb E[
Y_\theta^{(i,k)}\mid O_i=O].
\]
Then
\[
\widehat L^\star_{n,K}(\theta)
=
\frac1n\sum_{i=1}^n\overline Y_i.
\]

\begin{proof}[Proof of the unbiasedness part of
Theorem~\ref{thm:variance-decomposition}]
Conditional on $O_i$, the $K$ triples
\[
\left(
X_{-\Lambda_i}^{\star,(i,k)},
X_0^{(i,k)},
T^{(i,k)}
\right),
\qquad
k=1,\ldots,K,
\]
are conditionally i.i.d. with law
\[
r^\star_{\Lambda_i}
\left(
\cdot\mid X_{\Lambda_i}^{(i)}
\right)
\otimes
p_0
\otimes
\mathrm{Unif}[0,1].
\]
Hence
\[
\mathbb E[
Y_\theta^{(i,k)}\mid O_i]
=
\mu_\theta(O_i),
\]
and therefore
\[
\mathbb E[
\overline Y_i\mid O_i]
=
\mu_\theta(O_i).
\]
Taking expectations,
\begin{align*}
\mathbb E[\overline Y_i]
&=
\mathbb E_O[\mu_\theta(O)]
\\
&=
L_{\mathrm{MDFM}}(\theta)
\\
&=
L_{\mathrm{FM}}(\theta),
\end{align*}
where the last equality follows from
Theorem~\ref{thm:posterior-equivalence}. Consequently,
\[
\mathbb E[
\widehat L^\star_{n,K}(\theta)]
=
\frac1n\sum_{i=1}^n
\mathbb E[\overline Y_i]
=
L_{\mathrm{FM}}(\theta).
\]
\end{proof}

\begin{proof}[Proof of the variance-transfer identity in
Theorem~\ref{thm:variance-decomposition}]
Let
\[
X_1=(X_\Lambda,X_{-\Lambda})
\]
be constructed by first drawing
\[
O=(\Lambda,X_\Lambda)\sim P^\star_O
\]
and then drawing
\[
X_{-\Lambda}
\sim
r^\star_\Lambda(\cdot\mid X_\Lambda).
\]
By the defining property of the regular conditional distribution,
this reconstruction has marginal law
\[
X_1\sim p^\star.
\]

By definition of $g_\theta$,
\[
\mathbb E[
Y_\theta\mid O,X_{-\Lambda}]
=
g_\theta(X_1),
\]
and therefore
\[
\mathbb E[g_\theta(X_1)\mid O]
=
\mathbb E[Y_\theta\mid O]
=
\mu_\theta(O).
\]
Applying the law of total variance to $g_\theta(X_1)$ conditional
on $O$ gives
\begin{align*}
\operatorname{Var}_{X_1\sim p^\star}
[g_\theta(X_1)]
&=
\operatorname{Var}_O
\left(
\mathbb E[g_\theta(X_1)\mid O]
\right)
\\
&\quad+
\mathbb E_O
\left[
\operatorname{Var}
(g_\theta(X_1)\mid O)
\right]
\\
&=
\tau_\theta^2
+
\sigma_{\theta,\mathrm{miss}}^2.
\end{align*}
This proves the transfer identity.
\end{proof}

\begin{proof}[Proof of the variance decomposition in
Theorem~\ref{thm:variance-decomposition}]
Because observations and their associated Monte Carlo draws are
independent across $i$,
\[
\operatorname{Var}
\left(
\widehat L^\star_{n,K}(\theta)
\right)
=
\frac1n
\operatorname{Var}(\overline Y_1).
\]
Applying the law of total variance conditional on $O_1$,
\[
\operatorname{Var}(\overline Y_1)
=
\operatorname{Var}
\left(
\mathbb E[\overline Y_1\mid O_1]
\right)
+
\mathbb E
\left[
\operatorname{Var}(\overline Y_1\mid O_1)
\right].
\]
The first term equals
\[
\operatorname{Var}
(\mu_\theta(O_1))
=
\tau_\theta^2.
\]
Conditional on $O_1$, the $K$ losses are i.i.d., so
\[
\operatorname{Var}
(\overline Y_1\mid O_1)
=
\frac1K
\operatorname{Var}
(Y_\theta^{(1,1)}\mid O_1).
\]
Hence
\[
\operatorname{Var}(\overline Y_1)
=
\tau_\theta^2
+
\frac1K
\mathbb E
\left[
\operatorname{Var}
(Y_\theta^{(1,1)}\mid O_1)
\right].
\]

We now condition once more on the oracle completion. The conditional
law of total variance gives
\begin{align*}
&
\operatorname{Var}
(Y_\theta^{(1,1)}\mid O_1)
\\
&=
\mathbb E
\left[
\operatorname{Var}
(
Y_\theta^{(1,1)}
\mid
O_1,
X_{-\Lambda_1}^{\star,(1,1)}
)
\mid O_1
\right]
\\
&\quad+
\operatorname{Var}
\left(
\mathbb E[
Y_\theta^{(1,1)}
\mid
O_1,
X_{-\Lambda_1}^{\star,(1,1)}
]
\mid O_1
\right).
\end{align*}
Given the completed endpoint
$X_1^{\star,(1,1)}$, the only remaining randomness is
$(X_0,T)$. Therefore,
\[
\operatorname{Var}
(
Y_\theta^{(1,1)}
\mid
O_1,
X_{-\Lambda_1}^{\star,(1,1)}
)
=
\operatorname{Var}_{X_0,T}
\left(
\ell_\theta(X_0,X_1^{\star,(1,1)},T)
\mid
X_1^{\star,(1,1)}
\right),
\]
and
\[
\mathbb E[
Y_\theta^{(1,1)}
\mid
O_1,
X_{-\Lambda_1}^{\star,(1,1)}
]
=
g_\theta(X_1^{\star,(1,1)}).
\]
Taking expectation over $O_1$ and the completion, and using
\[
X_1^{\star,(1,1)}\sim p^\star,
\]
gives
\[
\mathbb E
\left[
\operatorname{Var}
(Y_\theta^{(1,1)}\mid O_1)
\right]
=
\sigma_{\theta,\mathrm{base}}^2
+
\sigma_{\theta,\mathrm{miss}}^2.
\]
Consequently,
\[
\operatorname{Var}(\overline Y_1)
=
\tau_\theta^2
+
\frac{
\sigma_{\theta,\mathrm{base}}^2+
\sigma_{\theta,\mathrm{miss}}^2
}{K},
\]
and hence
\[
\boxed{
\operatorname{Var}
\left(
\widehat L^\star_{n,K}(\theta)
\right)
=
\frac1n
\left[
\tau_\theta^2
+
\frac{
\sigma_{\theta,\mathrm{base}}^2+
\sigma_{\theta,\mathrm{miss}}^2
}{K}
\right].
}
\]
\end{proof}

\subsection{Consequences of the Variance Decomposition}
\label{app:variance-consequences}

\begin{proof}[Proof of Corollary~\ref{cor:variance-consequences}]
From the transfer identity,
\[
\tau_\theta^2
=
\operatorname{Var}_{X_1\sim p^\star}
[g_\theta(X_1)]
-
\sigma_{\theta,\mathrm{miss}}^2.
\]
Substituting this into
Theorem~\ref{thm:variance-decomposition},
\begin{align*}
\operatorname{Var}
(\widehat L^\star_{n,K})
&=
\frac1n
\left[
\operatorname{Var}(g_\theta(X_1))
-
\sigma_{\theta,\mathrm{miss}}^2
+
\frac{
\sigma_{\theta,\mathrm{base}}^2+
\sigma_{\theta,\mathrm{miss}}^2
}{K}
\right]
\\
&=
\frac1n
\left[
\operatorname{Var}(g_\theta(X_1))
-
\left(1-\frac1K\right)
\sigma_{\theta,\mathrm{miss}}^2
+
\frac{\sigma_{\theta,\mathrm{base}}^2}{K}
\right].
\end{align*}
Because
\[
\left(1-\frac1K\right)
\sigma_{\theta,\mathrm{miss}}^2\ge0,
\]
we obtain
\[
\operatorname{Var}
(\widehat L^\star_{n,K})
\le
\frac1n
\left[
\operatorname{Var}(g_\theta(X_1))
+
\frac{\sigma_{\theta,\mathrm{base}}^2}{K}
\right].
\]
At $K=1$,
\[
\operatorname{Var}
(\widehat L^\star_{n,1})
=
\frac1n
\left[
\operatorname{Var}(g_\theta(X_1))
+
\sigma_{\theta,\mathrm{base}}^2
\right],
\]
which is exactly the variance of the complete-data estimator using
$n$ independent endpoints and one independent $(X_0,T)$ draw per
endpoint.

For fixed $n$,
\[
K\mapsto
\frac1n
\left[
\tau_\theta^2
+
\frac{
\sigma_{\theta,\mathrm{base}}^2+
\sigma_{\theta,\mathrm{miss}}^2
}{K}
\right]
\]
is nonincreasing in $K$.

Under a fixed evaluation budget $N=nK$, substituting
$n=N/K$ gives
\[
\operatorname{Var}
(\widehat L^\star_{N/K,K})
=
\frac{
K\tau_\theta^2
+
\sigma_{\theta,\mathrm{base}}^2
+
\sigma_{\theta,\mathrm{miss}}^2
}{N}.
\]
This quantity is nondecreasing in $K$, and strictly increasing if
$\tau_\theta^2>0$. Hence $K=1$ minimizes the variance under this
fixed-budget model.
\end{proof}

\paragraph{Completion-only error.}
Define
\[
\overline\mu_n
=
\frac1n\sum_{i=1}^n\mu_\theta(O_i)
\]
and
\[
\widehat H_{n,K}
=
\frac1{nK}
\sum_{i=1}^n\sum_{k=1}^K
g_\theta(X_1^{\star,(i,k)}).
\]
Then
\[
\mathbb E[
\widehat H_{n,K}\mid O_{1:n}]
=
\overline\mu_n.
\]
Conditional independence of the completion draws gives
\begin{align*}
\operatorname{Var}
(
\widehat H_{n,K}
\mid
O_{1:n}
)
&=
\frac1{n^2K}
\sum_{i=1}^n
\operatorname{Var}
(
g_\theta(X_1^\star)\mid O_i
).
\end{align*}
Taking expectations,
\[
\boxed{
\mathbb E
\left[
(\widehat H_{n,K}-\overline\mu_n)^2
\right]
=
\frac{
\sigma_{\theta,\mathrm{miss}}^2
}{nK}.
}
\]
Combining this with
Theorem~\ref{thm:completion-variance} gives
\[
\mathbb E
\left[
(\widehat H_{n,K}-\overline\mu_n)^2
\right]
\le
\frac{
R^2L_g^2\rho d
}{nK}.
\]

Similarly, conditional on the observations,
\[
\mathbb E
\left[
\operatorname{Var}
(
\widehat L^\star_{n,K}
\mid
O_{1:n}
)
\right]
=
\frac{
\sigma_{\theta,\mathrm{base}}^2+
\sigma_{\theta,\mathrm{miss}}^2
}{nK}.
\]

\subsection{Stochastic-Gradient Variance Transfer}
\label{app:gradient-variance}

Let
\[
G_\theta(X_0,X_1,T)
=
\nabla_\theta\ell_\theta(X_0,X_1,T),
\qquad
\overline G_\theta(x)
=
\mathbb E_{X_0,T}
\left[
G_\theta(X_0,x,T)
\right].
\]
For an incomplete observation $O$, let
$m_\theta(O)=\mathbb E[G_\theta\mid O]$ and define
\[
\Sigma_{\theta,\mathrm{obs}}
=
\operatorname{Cov}_O(m_\theta(O)),
\]
\[
\Sigma_{\theta,\mathrm{base}}
=
\mathbb E_{X_1\sim p^\star}
\left[
\operatorname{Cov}_{X_0,T}
\left(
G_\theta\mid X_1
\right)
\right],
\]
and
\[
\Sigma_{\theta,\mathrm{miss}}
=
\mathbb E_O
\left[
\operatorname{Cov}
\left(
\overline G_\theta(X_1^\star)\mid O
\right)
\right].
\]

Let $\widehat G^\star_{n,K}$ be the average of the $nK$ stochastic-gradient
draws generated by the oracle completion procedure.

\begin{theorem}[Stochastic-gradient variance transfer]
\label{thm:gradient-variance}
Assume $\ell_\theta$ is differentiable in $\theta$, differentiation may be
interchanged with expectation, and
$\mathbb E\|\nabla_\theta\ell_\theta\|_2^2<\infty$. Then
\[
\mathbb E
\left[
\widehat G^\star_{n,K}
\right]
=
\nabla_\theta\mathcal L_{\mathrm{FM}}(\theta)
\]
and
\[
\operatorname{Cov}
\left(
\widehat G^\star_{n,K}
\right)
=
\frac{1}{n}
\left[
\Sigma_{\theta,\mathrm{obs}}
+
\frac{
\Sigma_{\theta,\mathrm{base}}
+
\Sigma_{\theta,\mathrm{miss}}
}{K}
\right].
\]
Moreover,
\[
\Sigma_{\theta,\mathrm{obs}}
+
\Sigma_{\theta,\mathrm{miss}}
=
\operatorname{Cov}_{X\sim p^\star}
\left(
\overline G_\theta(X)
\right).
\]
Thus $K=1$ matches the covariance of the corresponding complete-data
stochastic-gradient estimator. Under a fixed gradient-evaluation budget
$N=nK$, $K=1$ minimizes covariance in the Loewner order.
\end{theorem}

\begin{proof}[Proof of Theorem~\ref{thm:gradient-variance}]
For observation $i$, let
\[
\overline G_i
=
\frac1K
\sum_{k=1}^K
G_\theta^{(i,k)},
\]
so that
\[
\widehat G^\star_{n,K}
=
\frac1n\sum_{i=1}^n\overline G_i.
\]

By Corollary~\ref{cor:gradient-agreement},
\[
\mathbb E[
\widehat G^\star_{n,K}]
=
\nabla_\theta L_{\mathrm{FM}}(\theta).
\]

Independence across observations yields
\[
\operatorname{Cov}
(\widehat G^\star_{n,K})
=
\frac1n
\operatorname{Cov}(\overline G_1).
\]
Applying the law of total covariance conditional on $O_1$,
\[
\operatorname{Cov}(\overline G_1)
=
\operatorname{Cov}
\left(
\mathbb E[\overline G_1\mid O_1]
\right)
+
\mathbb E
\left[
\operatorname{Cov}
(\overline G_1\mid O_1)
\right].
\]
The first term is
\[
\Sigma_{\theta,\mathrm{obs}}.
\]
Conditional on $O_1$, the $K$ gradients are i.i.d., hence
\[
\operatorname{Cov}
(\overline G_1\mid O_1)
=
\frac1K
\operatorname{Cov}
(G_\theta^{(1,1)}\mid O_1).
\]

Condition once more on the completed endpoint $X_1^\star$. The
conditional law of total covariance gives
\begin{align*}
\operatorname{Cov}
(G_\theta\mid O)
&=
\mathbb E
\left[
\operatorname{Cov}
(G_\theta\mid X_1^\star,O)
\mid O
\right]
\\
&\quad+
\operatorname{Cov}
\left(
\mathbb E[G_\theta\mid X_1^\star,O]
\mid O
\right).
\end{align*}
Since $(X_0,T)$ is independent of $O$ conditional on the completed
endpoint,
\[
\operatorname{Cov}
(G_\theta\mid X_1^\star,O)
=
\operatorname{Cov}_{X_0,T}
(G_\theta\mid X_1^\star)
\]
and
\[
\mathbb E[
G_\theta\mid X_1^\star,O]
=
\overline G_\theta(X_1^\star).
\]
Taking expectation over $O$ and the oracle completion, using
$X_1^\star\sim p^\star$, gives
\[
\mathbb E_O[
\operatorname{Cov}(G_\theta\mid O)]
=
\Sigma_{\theta,\mathrm{base}}
+
\Sigma_{\theta,\mathrm{miss}}.
\]
Therefore
\[
\boxed{
\operatorname{Cov}
(\widehat G^\star_{n,K})
=
\frac1n
\left[
\Sigma_{\theta,\mathrm{obs}}
+
\frac{
\Sigma_{\theta,\mathrm{base}}
+
\Sigma_{\theta,\mathrm{miss}}
}{K}
\right].
}
\]

For the transfer identity, apply the law of total covariance to
$\overline G_\theta(X_1^\star)$ conditional on $O$:
\begin{align*}
\operatorname{Cov}_{X_1\sim p^\star}
(
\overline G_\theta(X_1)
)
&=
\operatorname{Cov}_O
\left(
\mathbb E[
\overline G_\theta(X_1^\star)\mid O]
\right)
\\
&\quad+
\mathbb E_O
\left[
\operatorname{Cov}
(
\overline G_\theta(X_1^\star)\mid O
)
\right]
\\
&=
\Sigma_{\theta,\mathrm{obs}}
+
\Sigma_{\theta,\mathrm{miss}}.
\end{align*}

For $K=1$,
\[
\operatorname{Cov}
(\widehat G^\star_{n,1})
=
\frac1n
\left[
\Sigma_{\theta,\mathrm{base}}
+
\operatorname{Cov}_{X_1\sim p^\star}
(
\overline G_\theta(X_1)
)
\right],
\]
which is precisely the covariance of the corresponding complete-data
gradient estimator.

Finally, if $N=nK$ is fixed,
\[
\operatorname{Cov}
(
\widehat G^\star_{N/K,K}
)
=
\frac1N
\left[
K\Sigma_{\theta,\mathrm{obs}}
+
\Sigma_{\theta,\mathrm{base}}
+
\Sigma_{\theta,\mathrm{miss}}
\right].
\]
Thus, for $K_2\ge K_1$,
\[
\operatorname{Cov}
(
\widehat G^\star_{N/K_2,K_2}
)
-
\operatorname{Cov}
(
\widehat G^\star_{N/K_1,K_1}
)
=
\frac{K_2-K_1}{N}
\Sigma_{\theta,\mathrm{obs}}
\succeq0.
\]
Hence $K=1$ minimizes covariance in the Loewner order under the
fixed-gradient-evaluation-budget model.

Since the estimator is unbiased,
\[
\mathbb E
\left[
\|
\widehat G^\star_{n,K}
-
\nabla_\theta L_{\mathrm{FM}}(\theta)
\|_2^2
\right]
=
\operatorname{tr}
\operatorname{Cov}
(
\widehat G^\star_{n,K}
).
\]
\end{proof}

\subsection{Completion Variance Bound}
\label{app:completion-variance}

\begin{proof}[Proof of Theorem~\ref{thm:completion-variance}]
Fix an observation
\[
O=(\Lambda,X_\Lambda)
\]
and let
\[
m=|-\Lambda|.
\]
Conditional on $O$, only the missing coordinates vary. Let
$u,v\in[-R,R]^m$ be two admissible values of the missing block, and
write
\[
x_1(u)=(X_\Lambda,u),
\qquad
x_1(v)=(X_\Lambda,v).
\]
Then
\[
\|x_1(u)-x_1(v)\|_2^2
=
\sum_{j\in-\Lambda}
(u_j-v_j)^2
\le
4R^2m,
\]
so
\[
\|x_1(u)-x_1(v)\|_2
\le
2R\sqrt m.
\]
By the endpoint Lipschitz assumption,
\[
|g_\theta(x_1(u))-g_\theta(x_1(v))|
\le
2RL_g\sqrt m.
\]

Let
\[
Z
=
g_\theta(X_\Lambda,X_{-\Lambda}),
\qquad
X_{-\Lambda}
\sim
r^\star_\Lambda(\cdot\mid X_\Lambda).
\]
Conditional on $O$, the essential range of $Z$ therefore has width at
most
\[
2RL_g\sqrt m.
\]
Popoviciu's inequality gives
\[
\operatorname{Var}(Z\mid O)
\le
\frac{
(2RL_g\sqrt m)^2
}{4}
=
R^2L_g^2m.
\]
Taking expectations over $O$,
\begin{align*}
\sigma_{\theta,\mathrm{miss}}^2
&=
\mathbb E_O[
\operatorname{Var}(Z\mid O)]
\\
&\le
R^2L_g^2
\mathbb E[m(\Lambda)]
\\
&=
R^2L_g^2\rho d.
\end{align*}
\end{proof}

\subsection{Oracle Concentration}
\label{app:oracle-concentration}

Define
\[
V_{\rho,K}
=
\min
\left\{
\tau_\theta^2
+
\frac{
\sigma_{\theta,\mathrm{base}}^2+
R^2L_g^2\rho d
}{K},
\quad
\operatorname{Var}(g_\theta(X_1))
+
\frac{\sigma_{\theta,\mathrm{base}}^2}{K}
\right\}.
\]

\begin{theorem}[Oracle Bernstein bound]
\label{thm:posterior-equivalence-bernstein}
Fix $\theta$. Suppose the conditions of
Theorem~\ref{thm:variance-decomposition} and
Assumption~\ref{ass:bdd} hold. Assume additionally that
\[
0
\le
\ell_\theta(X_0,X_1^\star,T)
\le
B
\qquad
\text{a.s.}
\]
for some $B<\infty$. Then, for every $\delta\in(0,1)$, with probability at
least $1-\delta$,
\[
\left|
\widehat{\mathcal L}_{n,K}^\star(\theta)
-
\mathcal L_{\mathrm{FM}}(\theta)
\right|
\le
\sqrt{
\frac{
2V_{\rho,K}\log(2/\delta)
}{n}
}
+
\frac{
2B\log(2/\delta)
}{3n}.
\]
\end{theorem}

\begin{proof}[Proof of Theorem~\ref{thm:posterior-equivalence-bernstein}]
Let
\[
Z_i=\overline Y_i,
\qquad i=1,\ldots,n.
\]
The $Z_i$ are i.i.d. because the incomplete observations are i.i.d.
and all completion, base-point, and time draws are independent across
observations.

By Theorem~\ref{thm:variance-decomposition},
\[
\mathbb E[Z_i]
=
L_{\mathrm{FM}}(\theta).
\]
The bounded-loss assumption gives
\[
0\le Z_i\le B
\qquad\text{a.s.}
\]
and hence
\[
|Z_i-L_{\mathrm{FM}}(\theta)|
\le B.
\]

Again by Theorem~\ref{thm:variance-decomposition},
\[
\operatorname{Var}(Z_i)
=
\tau_\theta^2
+
\frac{
\sigma_{\theta,\mathrm{base}}^2+
\sigma_{\theta,\mathrm{miss}}^2
}{K}.
\]
Theorem~\ref{thm:completion-variance} gives
\[
\sigma_{\theta,\mathrm{miss}}^2
\le
R^2L_g^2\rho d,
\]
which yields the first branch of $V_{\rho,K}$.

For the second branch, the transfer identity implies
\[
\tau_\theta^2+\sigma_{\theta,\mathrm{miss}}^2
=
\operatorname{Var}(g_\theta(X_1)).
\]
Therefore
\begin{align*}
\tau_\theta^2
+
\frac{\sigma_{\theta,\mathrm{miss}}^2}{K}
&=
\tau_\theta^2+\sigma_{\theta,\mathrm{miss}}^2
-
\left(1-\frac1K\right)
\sigma_{\theta,\mathrm{miss}}^2
\\
&\le
\operatorname{Var}(g_\theta(X_1)),
\end{align*}
so
\[
\operatorname{Var}(Z_i)
\le
\operatorname{Var}(g_\theta(X_1))
+
\frac{\sigma_{\theta,\mathrm{base}}^2}{K}.
\]
Hence
\[
\operatorname{Var}(Z_i)\le V_{\rho,K}.
\]

Applying Bernstein's inequality to
\[
\widehat L^\star_{n,K}(\theta)
=
\frac1n\sum_{i=1}^nZ_i
\]
gives, with probability at least $1-\delta$,
\[
\boxed{
\left|
\widehat L^\star_{n,K}(\theta)
-
L_{\mathrm{FM}}(\theta)
\right|
\le
\sqrt{
\frac{
2V_{\rho,K}\log(2/\delta)
}{n}
}
+
\frac{
2B\log(2/\delta)
}{3n}.
}
\]
\end{proof}

\subsection{Completion-Bias Bound}
\label{app:completion-bias}

\begin{proof}[Proof of Lemma~\ref{lem:completion-bias}]
Fix
\[
O=(\Lambda,X_\Lambda)
\]
and define the admissible missing-coordinate section
\[
\mathcal X_O
=
\left\{
u:
(X_\Lambda,u)\in\mathcal X
\right\}.
\]
Define
\[
h_O(u)
=
g_\theta(X_\Lambda,u).
\]
Since $g_\theta$ is $L_g$-Lipschitz on $\mathcal X$,
for all $u,v\in\mathcal X_O$,
\[
|h_O(u)-h_O(v)|
\le
L_g\|u-v\|_2.
\]
Because $\mathcal X$ is convex, $\mathcal X_O$ is convex.
By the McShane extension theorem, $h_O$ admits an
$L_g$-Lipschitz extension to the ambient Euclidean space. Thus
Kantorovich--Rubinstein duality yields
\[
\left|
\int h_O\,dq_\phi(\cdot\mid O)
-
\int h_O\,dr^\star_\Lambda(\cdot\mid X_\Lambda)
\right|
\le
L_g
W_1
\left(
q_\phi(\cdot\mid O),
r^\star_\Lambda(\cdot\mid X_\Lambda)
\right).
\]

By Theorem~\ref{thm:posterior-equivalence},
\[
L_{\mathrm{FM}}(\theta)
=
\mathbb E_O
\left[
\int
h_O(u)
\,dr^\star_\Lambda(u\mid X_\Lambda)
\right],
\]
while
\[
L_\phi(\theta)
=
\mathbb E_O
\left[
\int
h_O(u)
\,dq_\phi(u\mid O)
\right].
\]
Hence
\begin{align*}
&
|L_\phi(\theta)-L_{\mathrm{FM}}(\theta)|
\\
&\le
\mathbb E_O
\left[
\left|
\int h_O\,dq_\phi
-
\int h_O\,dr^\star_\Lambda
\right|
\right]
\\
&\le
L_g
\mathbb E_O
W_1
\left(
q_\phi(\cdot\mid O),
r^\star_\Lambda(\cdot\mid X_\Lambda)
\right).
\end{align*}
Under MCAR or MAR,
Corollary~\ref{cor:mcar-mar} additionally gives
\[
r^\star_\Lambda(\cdot\mid X_\Lambda)
=
p^\star(\cdot\mid X_\Lambda)
\qquad\text{a.s.}
\]
\end{proof}

\subsection{Concentration for the Learned Completion Estimator}
\label{app:learned-concentration}

Condition on a fixed completion model $q_\phi$. Define
\[
\mu_{\theta,\phi}(O)
=
\mathbb E_{q_\phi,X_0,T}
[
Y^\phi_\theta\mid O],
\]
\[
\tau_{\theta,\phi}^2
=
\operatorname{Var}_O
(
\mu_{\theta,\phi}(O)
),
\]
\[
\sigma_{\theta,\phi,\mathrm{comp}}^2
=
\mathbb E_O
\left[
\operatorname{Var}_{
\widetilde X_{-\Lambda}\sim q_\phi(\cdot\mid O)
}
\left(
g_\theta(\widetilde X_1)
\mid O
\right)
\right],
\]
and
\[
\sigma_{\theta,\phi,\mathrm{base}}^2
=
\mathbb E_{O,\widetilde X_1}
\left[
\operatorname{Var}_{X_0,T}
\left(
\ell_\theta(X_0,\widetilde X_1,T)
\mid
\widetilde X_1
\right)
\right].
\]
Set
\[
V_{\theta,\phi,K}
=
\tau_{\theta,\phi}^2
+
\frac{
\sigma_{\theta,\phi,\mathrm{base}}^2+
\sigma_{\theta,\phi,\mathrm{comp}}^2
}{K}.
\]

\begin{theorem}[Learned-completion Bernstein bound]
\label{thm:learned-bernstein}
Condition on a fixed completion model $q_\phi$ that is independent of the
evaluation sample. Assume that
\[
0
\le
\ell_\theta(X_0,\widetilde X_1,T)
\le
B_\phi
\qquad
\text{a.s.}
\]
and that $V_{\theta,\phi,K}<\infty$. Then, for every
$\delta\in(0,1)$, with probability at least $1-\delta$,
\[
\left|
\widehat{\mathcal L}_{n,K}^{\phi}(\theta)
-
\mathcal L_\phi(\theta)
\right|
\le
\sqrt{
\frac{
2V_{\theta,\phi,K}\log(2/\delta)
}{n}
}
+
\frac{
2B_\phi\log(2/\delta)
}{3n}.
\]
\end{theorem}

\begin{proof}[Proof of Theorem~\ref{thm:learned-bernstein}]
All probabilities and expectations below are conditional on the
fitted $q_\phi$, which is therefore treated as fixed.

For observation $i$, let
\[
\overline Y_i^\phi
=
\frac1K
\sum_{k=1}^K
Y_{\theta,\phi}^{(i,k)}.
\]
The observations are i.i.d., and conditional on each observation the
$K$ learned-completion draws are i.i.d. Therefore the same
two-stage total-variance calculation as in
Theorem~\ref{thm:variance-decomposition} gives
\[
\operatorname{Var}
(\overline Y_i^\phi)
=
V_{\theta,\phi,K}.
\]
Also,
\[
\mathbb E[
\overline Y_i^\phi]
=
L_\phi(\theta).
\]
By hypothesis,
\[
0\le
\overline Y_i^\phi
\le
B_\phi
\qquad\text{a.s.}
\]
Therefore Bernstein's inequality gives, with probability at least
$1-\delta$,
\[
\boxed{
\left|
\widehat L^\phi_{n,K}(\theta)
-
L_\phi(\theta)
\right|
\le
\sqrt{
\frac{
2V_{\theta,\phi,K}\log(2/\delta)
}{n}
}
+
\frac{
2B_\phi\log(2/\delta)
}{3n}.
}
\]
\end{proof}


\begin{proof}[Proof of Corollary~\ref{cor:learned-total}]
By the triangle inequality,
\begin{align*}
&
\left|
\widehat L^\phi_{n,K}(\theta)
-
L_{\mathrm{FM}}(\theta)
\right|
\\
&\le
\left|
\widehat L^\phi_{n,K}(\theta)
-
L_\phi(\theta)
\right|
+
\left|
L_\phi(\theta)-L_{\mathrm{FM}}(\theta)
\right|.
\end{align*}
Apply Theorem~\ref{thm:learned-bernstein} to the first term and
Lemma~\ref{lem:completion-bias} to the second. With probability at
least $1-\delta$,
\[
\boxed{
\begin{aligned}
&
\left|
\widehat L^\phi_{n,K}(\theta)
-
L_{\mathrm{FM}}(\theta)
\right|
\\
&\le
\sqrt{
\frac{
2V_{\theta,\phi,K}\log(2/\delta)
}{n}
}
+
\frac{
2B_\phi\log(2/\delta)
}{3n}
\\
&\quad+
L_g
\mathbb E_O
W_1
\left(
q_\phi(\cdot\mid O),
r^\star_\Lambda(\cdot\mid X_\Lambda)
\right).
\end{aligned}
}
\]
\end{proof}

\subsection{Completion Error and Population Field Error}
\label{app:completion-field}

We first record the squared-loss regression identity used in
Proposition~\ref{prop:completion-field}. Let
\[
U=X_1-X_0,
\qquad
Z=(X_T,T),
\]
and define
\[
u(X_T,T)
=
\mathbb E[U\mid Z].
\]
For any measurable vector field $v$ with finite second moment,
\begin{align*}
&
\mathbb E
\|v(X_T,T)-U\|_2^2
\\
&=
\mathbb E
\|v(X_T,T)-u(X_T,T)\|_2^2
+
\mathbb E
\|u(X_T,T)-U\|_2^2,
\end{align*}
because the cross term vanishes:
\begin{align*}
&
\mathbb E
\left[
\left\langle
v(X_T,T)-u(X_T,T),
u(X_T,T)-U
\right\rangle
\right]
\\
&=
\mathbb E
\left[
\left\langle
v(X_T,T)-u(X_T,T),
\mathbb E[u(X_T,T)-U\mid Z]
\right\rangle
\right]
\\
&=0.
\end{align*}
Consequently,
\[
L_{\mathrm{FM}}(\theta)
-
L_{\mathrm{FM}}(u)
=
\int_0^1
\mathbb E
\left[
\|v_\theta(X_t,t)-u(X_t,t)\|_2^2
\right]dt.
\]

\begin{proof}[Proof of Proposition~\ref{prop:completion-field}]
By the uniform endpoint-loss Lipschitz assumption and
Lemma~\ref{lem:completion-bias}, for every $\theta\in\Theta$,
\[
|L_\phi(\theta)-L_{\mathrm{FM}}(\theta)|
\le
\overline L_g\varepsilon_\phi.
\]
Therefore
\begin{align*}
L_{\mathrm{FM}}(\theta_\phi)
&\le
L_\phi(\theta_\phi)
+
\overline L_g\varepsilon_\phi
\\
&\le
\inf_{\theta\in\Theta}L_\phi(\theta)
+
\varepsilon_{\mathrm{opt}}
+
\overline L_g\varepsilon_\phi.
\end{align*}
Moreover, for every $\theta\in\Theta$,
\[
L_\phi(\theta)
\le
L_{\mathrm{FM}}(\theta)
+
\overline L_g\varepsilon_\phi,
\]
so
\[
\inf_{\theta\in\Theta}L_\phi(\theta)
\le
\inf_{\theta\in\Theta}L_{\mathrm{FM}}(\theta)
+
\overline L_g\varepsilon_\phi.
\]
Combining,
\[
L_{\mathrm{FM}}(\theta_\phi)
\le
\inf_{\theta\in\Theta}
L_{\mathrm{FM}}(\theta)
+
2\overline L_g\varepsilon_\phi
+
\varepsilon_{\mathrm{opt}}.
\]
Subtracting $L_{\mathrm{FM}}(u)$ gives
\[
L_{\mathrm{FM}}(\theta_\phi)
-
L_{\mathrm{FM}}(u)
\le
\mathcal A_\Theta
+
2\overline L_g\varepsilon_\phi
+
\varepsilon_{\mathrm{opt}}.
\]
Using the squared-loss identity above,
\[
\boxed{
\int_0^1
\mathbb E
\left[
\|v_{\theta_\phi}(X_t,t)-u(X_t,t)\|_2^2
\right]dt
\le
\mathcal A_\Theta
+
2\overline L_g\varepsilon_\phi
+
\varepsilon_{\mathrm{opt}}.
}
\]
\end{proof}

\subsection{Terminal Distribution Bound}
\label{app:terminal-bound}

\begin{proof}[Proof of Theorem~\ref{thm:completion-generation}]
For the straight-line interpolant
\[
X_t=(1-t)X_0+tX_1,
\]
we have
\[
\dot X_t=X_1-X_0.
\]
Its marginal flow-matching velocity is therefore
\[
u(x,t)
=
\mathbb E[
X_1-X_0\mid X_t=x].
\]

By Proposition~\ref{prop:completion-field},
\[
\int_0^1
\mathbb E
\left[
\|v_{\theta_\phi}(X_t,t)-u(X_t,t)\|_2^2
\right]dt
\le
\varepsilon^2,
\]
where
\[
\varepsilon^2
:=
\mathcal A_\Theta
+
2\overline L_g\varepsilon_\phi
+
\varepsilon_{\mathrm{opt}}.
\]

By the assumptions of
Theorem~\ref{thm:completion-generation}, the hypotheses of
Theorem~1 of \citet{benton2024error} hold for the true marginal
flow and the learned flow. Their result therefore gives
\[
W_2
\left(
p_1^{\theta_\phi},
p^\star
\right)
\le
\varepsilon
\exp
\left\{
\int_0^1L_t\,dt
\right\}.
\]
With
\[
C_{\mathrm{flow}}
=
\exp
\left\{
\int_0^1L_t\,dt
\right\},
\]
we conclude
\[
\boxed{
W_2
\left(
p_1^{\theta_\phi},
p^\star
\right)
\le
C_{\mathrm{flow}}
\sqrt{
\mathcal A_\Theta
+
2\overline L_g\varepsilon_\phi
+
\varepsilon_{\mathrm{opt}}
}.
}
\]
\end{proof}

\subsection{A Lipschitz-Controlled Model Class}
\label{app:lipschitz-class}

We give sufficient conditions that make the constants in
Theorem~\ref{thm:completion-generation} explicit.

Assume
\[
\|X_0\|_2\le R_0
\qquad\text{a.s.}
\]
and, under Assumption~\ref{ass:bdd},
\[
\|X_1\|_2
\le
R_1:=R\sqrt d
\qquad\text{a.s.}
\]
Assume also that learned completed endpoints satisfy
\[
\|\widetilde X_1\|_2
\le
R_q
\qquad
q_\phi(\cdot\mid O)\text{-a.s.}
\]
for some $R_q<\infty$. Define
\[
R_E:=\max\{R_1,R_q\},
\qquad
R_X:=\max\{R_0,R_E\},
\qquad
V_0:=R_0+R_E.
\]

For constants $L,M_0<\infty$, let
\[
\Theta_{L,M_0}
:=
\left\{
\theta:
\operatorname{Lip}_x(v_\theta(\cdot,t))\le L
\text{ for all }t,
\quad
\sup_t\|v_\theta(0,t)\|_2\le M_0
\right\}.
\]
Along every interpolation path,
\[
\|X_t\|_2\le R_X
\]
and therefore
\[
\|v_\theta(X_t,t)\|_2
\le
M_0+LR_X.
\]
Define
\[
C_{L,M_0}
:=
M_0+LR_X+V_0.
\]
Then
\[
0
\le
\ell_\theta(X_0,X_1,T)
\le
C_{L,M_0}^2
\qquad
\text{a.s.}
\]
and the endpoint loss satisfies
\[
L_{g_\theta}
\le
2C_{L,M_0}(L+1).
\]

Define
\[
\mathcal A_{L,M_0}
:=
\inf_{\theta\in\Theta_{L,M_0}}
\int_0^1
\mathbb E
\left[
\|v_\theta(X_t,t)-u(X_t,t)\|_2^2
\right]dt.
\]

\begin{corollary}[Regularity and approximation tradeoff]
\label{cor:lipschitz-class}
Suppose the hypotheses of
Theorem~\ref{thm:completion-generation} hold for a learned field in
$\Theta_{L,M_0}$. Then
\[
W_2
\left(
p_1^{\theta_{\phi,L}},
p^\star
\right)
\le
e^L
\sqrt{
\mathcal A_{L,M_0}
+
4C_{L,M_0}(L+1)\varepsilon_\phi
+
\varepsilon_{\mathrm{opt}}
}.
\]
\end{corollary}

A smaller spatial Lipschitz constant improves the explicit flow-stability
factor but may increase the approximation error
$\mathcal A_{L,M_0}$. The bound therefore reflects a regularity and
approximation tradeoff rather than implying that smaller $L$ is always better.

\subsection{Scope of the Guarantees}
\label{app:scope}

The variance results are pointwise in the parameter $\theta$. The
stochastic-gradient result therefore does not imply that $K=1$ minimizes the
error, convergence time, or final parameter error of a multi-step optimization
trajectory.

The learned-completion concentration result conditions on a fixed $q_\phi$
that is independent of the evaluation sample. It does not directly cover the
dependence created when $q_\phi$ and $v_\theta$ are estimated from the same
finite dataset.

Proposition~\ref{prop:completion-field} and
Theorem~\ref{thm:completion-generation} concern an approximate population
minimizer of $\mathcal L_\phi$. They do not provide a finite-sample
generalization or SGD-convergence guarantee for the parameter returned by the
training procedure. The terminal-distribution bound also concerns exact ODE
integration, so numerical discretization error is separate.

Finally, the oracle estimator identities are mechanism-independent, whereas
identification and estimation of the required completion posterior are not.
Under MAR, stabilized inverse-probability weighting corrects the
self-supervision selection tilt only when the propensity is identifiable and
adequate positivity holds. It does not resolve insufficient pattern coverage.
Under MNAR, additional identifying assumptions are generally required.

\section{Mechanism-Specific Consequences of Posterior Completion}
\label{app:mechanisms}

This appendix makes explicit how MCAR, MAR, and MNAR affect the
completion distribution used by Missing-Data Flow Matching. The central
oracle identity itself is mechanism-independent: as shown in
Theorem~\ref{thm:posterior-equivalence}, drawing the missing coordinates
from their true pattern-conditioned posterior reconstructs the
complete-data flow-matching objective under any missingness mechanism.
The mechanisms differ instead in how this posterior relates to the
ordinary complete-data conditional distribution and in whether the
required completion law can be identified and estimated from observed
data.

We also characterize the objective obtained when the ordinary
complete-data conditional is substituted for the pattern-conditioned
posterior. Under MCAR and MAR this substitution is correct at the
population level, whereas under MNAR it generally induces an explicit
reweighting of the complete-data objective.

\subsection{Setup}

Let $P^\star$ denote the joint law of $(X,\Lambda)$, where
\[
X\sim p^\star
\]
on $\mathbb R^d$ and $\Lambda\subseteq[d]$ is the random set of
observed coordinates. For each pattern $\lambda\subseteq[d]$, define
the full-data pattern propensity
\[
\pi_\lambda(x)
:=
P^\star(\Lambda=\lambda\mid X=x).
\]
An incomplete observation is
\[
O=(\Lambda,X_\Lambda).
\]
Throughout this appendix, $g_\theta$ denotes the endpoint loss defined
in Section~3, and we assume
\[
\mathbb E_{X\sim p^\star}
\left[
|g_\theta(X)|
\right]
<\infty.
\]

For the explicit density formulas below, we additionally assume that
$p^\star$ and the relevant marginal and conditional distributions admit
densities with respect to suitable dominating measures. The
corresponding statements may equivalently be formulated using regular
conditional probabilities and Radon--Nikodym derivatives.

\subsection{The Three Missingness Mechanisms}

Following the standard terminology
\citep{rubin1976inference,little2019statistical}, the mechanisms are
defined patternwise as follows.

\begin{itemize}
    \item \textbf{Missing completely at random (MCAR).}
    For every supported pattern $\lambda$,
    \[
    \pi_\lambda(x)
    =
    P^\star(\Lambda=\lambda)
    \qquad
    p^\star\text{-a.e. }x.
    \]
    Equivalently, $\Lambda$ is independent of $X$.

    \item \textbf{Missing at random (MAR).}
    For every supported pattern $\lambda$, there exists a measurable
    function, also denoted by $\pi_\lambda$, such that
    \[
    P^\star(\Lambda=\lambda\mid X=x)
    =
    \pi_\lambda(x_\lambda)
    \qquad
    p^\star\text{-a.e. }x.
    \]
    Thus, conditional on pattern $\lambda$, its probability can depend
    on the coordinates that would be observed under that pattern, but
    not on the corresponding missing coordinates.

    \item \textbf{Missing not at random (MNAR).}
    The MAR condition fails for at least one supported pattern, so the
    corresponding propensity may depend on coordinates in
    $x_{-\lambda}$.
\end{itemize}

MCAR is a special case of MAR.

\subsection{The Pattern-Conditioned Completion Posterior}

For every supported pattern $\lambda$, recall the oracle completion law
\[
r^\star_\lambda(\cdot\mid x_\lambda)
:=
P^\star
\left(
X_{-\lambda}\in\cdot
\mid
X_\lambda=x_\lambda,\Lambda=\lambda
\right).
\]

To characterize its relation to the ordinary complete-data conditional,
define
\[
\bar\pi_\lambda(x_\lambda)
:=
P^\star
\left(
\Lambda=\lambda
\mid
X_\lambda=x_\lambda
\right).
\]
By the tower property,
\[
\bar\pi_\lambda(x_\lambda)
=
\mathbb E
\left[
\pi_\lambda(X)
\mid
X_\lambda=x_\lambda
\right].
\]

\begin{lemma}[Posterior factorization]
\label{lem:bayes}
Fix a pattern $\lambda$ with
\[
P^\star(\Lambda=\lambda)>0.
\]
For
$P^\star_{X_\lambda\mid\Lambda=\lambda}$-almost every
$x_\lambda$, so that
$\bar\pi_\lambda(x_\lambda)>0$, the pattern-conditioned
posterior has density
\[
\boxed{
r^\star_\lambda(x_{-\lambda}\mid x_\lambda)
=
p^\star_{X_{-\lambda}\mid X_\lambda}
(x_{-\lambda}\mid x_\lambda)
\frac{
\pi_\lambda(x_\lambda,x_{-\lambda})
}{
\bar\pi_\lambda(x_\lambda)
}.
}
\]
The identity holds for
$p^\star_{X_{-\lambda}\mid X_\lambda=x_\lambda}$-almost every
$x_{-\lambda}$.
\end{lemma}

\begin{proof}
The joint density of
$(X,\Lambda=\lambda)$ is
\[
p^\star(x_\lambda,x_{-\lambda})
\pi_\lambda(x_\lambda,x_{-\lambda}).
\]
Integrating over $x_{-\lambda}$ shows that the joint density of
$(X_\lambda,\Lambda=\lambda)$ is
\[
p^\star_{X_\lambda}(x_\lambda)
\bar\pi_\lambda(x_\lambda).
\]
Hence Bayes' rule gives
\[
r^\star_\lambda(x_{-\lambda}\mid x_\lambda)
=
\frac{
p^\star(x_\lambda,x_{-\lambda})
\pi_\lambda(x_\lambda,x_{-\lambda})
}{
p^\star_{X_\lambda}(x_\lambda)
\bar\pi_\lambda(x_\lambda)
}.
\]
Using
\[
p^\star(x_\lambda,x_{-\lambda})
=
p^\star_{X_\lambda}(x_\lambda)
p^\star_{X_{-\lambda}\mid X_\lambda}
(x_{-\lambda}\mid x_\lambda)
\]
gives the stated factorization.
\end{proof}

The multiplicative factor
\[
\frac{
\pi_\lambda(x_\lambda,x_{-\lambda})
}{
\bar\pi_\lambda(x_\lambda)
}
\]
is the selection tilt induced by conditioning on the realized pattern.

Under MAR,
\[
\pi_\lambda(x_\lambda,x_{-\lambda})
=
\pi_\lambda(x_\lambda),
\]
and therefore
\[
\bar\pi_\lambda(x_\lambda)
=
\pi_\lambda(x_\lambda).
\]
Consequently,
\[
\boxed{
r^\star_\lambda(\cdot\mid x_\lambda)
=
p^\star_{X_{-\lambda}\mid X_\lambda}
(\cdot\mid x_\lambda)
}
\]
almost everywhere. MCAR is the special case in which the common
propensity is constant in $x_\lambda$.

Under MNAR the numerator can genuinely depend on
$x_{-\lambda}$, so the pattern-conditioned posterior generally differs
from the ordinary complete-data conditional.

\subsection{Oracle Equivalence for an Arbitrary Mechanism}

The preceding factorization is not needed for the oracle equivalence
itself. By definition of the regular conditional distribution,
\[
\mathbb E_{
X_{-\Lambda}^\star
\sim
r^\star_\Lambda(\cdot\mid X_\Lambda)
}
\left[
g_\theta(X_\Lambda,X_{-\Lambda}^\star)
\mid O
\right]
=
\mathbb E
\left[
g_\theta(X)
\mid O
\right].
\]
Therefore,
\begin{align*}
&
\mathbb E_O
\mathbb E_{
X_{-\Lambda}^\star
\sim
r^\star_\Lambda(\cdot\mid X_\Lambda)
}
\left[
g_\theta(X_\Lambda,X_{-\Lambda}^\star)
\right]
\\
&\qquad=
\mathbb E_O
\left[
\mathbb E[g_\theta(X)\mid O]
\right]
\\
&\qquad=
\mathbb E_{X\sim p^\star}
[g_\theta(X)]
\\
&\qquad=
\mathcal L_{\mathrm{FM}}(\theta).
\end{align*}
This is precisely
Theorem~\ref{thm:posterior-equivalence}. No MCAR or MAR assumption is
used.

Thus the distinction among the missingness mechanisms is not whether
the oracle identity is valid, but whether the required posterior
$r^\star_\Lambda$ can be represented, identified, and estimated from
the available observations.

\subsection{Using the Ordinary Complete-Data Conditional}
\label{app:ordinary-conditional}

It is useful to characterize what happens if one uses the ordinary
complete-data conditional instead of the pattern-conditioned posterior.
Define
\begin{equation}
\label{eq:conditional-surrogate}
\mathcal L_{\mathrm{cond}}(\theta)
:=
\mathbb E_O
\mathbb E_{
X_{-\Lambda}^\dagger
\sim
p^\star_{X_{-\Lambda}\mid X_\Lambda}
(\cdot\mid X_\Lambda)
}
\left[
g_\theta(X_\Lambda,X_{-\Lambda}^\dagger)
\right].
\end{equation}
The superscript $\dagger$ distinguishes this surrogate completion from
the latent missing coordinates of the original observation and from
the oracle posterior draw $X_{-\Lambda}^\star$.

We restate Proposition~\ref{prop:naive} for convenience:
$\mathcal{L}_{\mathrm{cond}}(\theta) = \mathbb{E}_{X\sim p^\star}[w(X) g_\theta(X)]$
with $w(x) = \sum_{\lambda} \bar{\pi}_\lambda(x_\lambda)$.

\begin{proof}[Proof of Proposition~\ref{prop:naive}]
For every supported pattern $\lambda$,
\[
p(x_\lambda\mid\Lambda=\lambda)
=
\frac{
p^\star_{X_\lambda}(x_\lambda)
\bar\pi_\lambda(x_\lambda)
}{
P^\star(\Lambda=\lambda)
}.
\]
Expanding~\eqref{eq:conditional-surrogate} over patterns gives
\begin{align*}
\mathcal L_{\mathrm{cond}}(\theta)
&=
\sum_\lambda
P^\star(\Lambda=\lambda)
\int
p(x_\lambda\mid\Lambda=\lambda)
\\
&\qquad\qquad\times
\int
p^\star_{X_{-\lambda}\mid X_\lambda}
(x_{-\lambda}\mid x_\lambda)
g_\theta(x)
\,dx_{-\lambda}\,dx_\lambda
\\
&=
\sum_\lambda
\int
p^\star(x)
\bar\pi_\lambda(x_\lambda)
g_\theta(x)\,dx
\\
&=
\mathbb E_{X\sim p^\star}
\left[
w(X)g_\theta(X)
\right].
\end{align*}

Nonnegativity is immediate. Furthermore,
\[
\mathbb E_{p^\star}
[
\bar\pi_\lambda(X_\lambda)
]
=
P^\star(\Lambda=\lambda)
\]
by the tower property. Summing over $\lambda$ gives
\[
\mathbb E_{p^\star}[w(X)]
=
\sum_\lambda
P^\star(\Lambda=\lambda)
=
1.
\]

Under MAR,
\[
\bar\pi_\lambda(X_\lambda)
=
\pi_\lambda(X)
\qquad
p^\star\text{-a.s.},
\]
so
\[
w(X)
=
\sum_\lambda
\pi_\lambda(X)
=
1
\qquad
p^\star\text{-a.s.}
\]
This proves the final claim.
\end{proof}

Under MNAR, $w$ need not equal one. The exact objective difference is
therefore
\begin{equation}
\label{eq:mnar-surrogate-gap}
\mathcal L_{\mathrm{cond}}(\theta)
-
\mathcal L_{\mathrm{FM}}(\theta)
=
\mathbb E_{p^\star}
\left[
(w(X)-1)g_\theta(X)
\right].
\end{equation}
Since
\[
\mathbb E_{p^\star}[w(X)-1]=0,
\]
the gap can vanish even when $w$ is nonconstant; no fixed sign is
implied without additional assumptions relating $w$ and $g_\theta$.

\subsection{Patternwise Versus Aggregate Equality under MAR}

Proposition~\ref{prop:naive} explains an important distinction between
MCAR and MAR.

Under MCAR,
\[
P^\star(X_\lambda\in dx_\lambda\mid\Lambda=\lambda)
=
p^\star_{X_\lambda}(dx_\lambda),
\]
so for every supported pattern $\lambda$,
\[
\mathbb E
\left[
g_\theta(X_\lambda,X_{-\lambda}^\dagger)
\mid
\Lambda=\lambda
\right]
=
\mathcal L_{\mathrm{FM}}(\theta),
\]
where
\[
X_{-\lambda}^\dagger
\sim
p^\star_{X_{-\lambda}\mid X_\lambda}
(\cdot\mid X_\lambda).
\]
Thus MCAR gives equality pattern by pattern.

Under MAR,
\[
p(x_\lambda\mid\Lambda=\lambda)
=
\frac{
p^\star_{X_\lambda}(x_\lambda)
\pi_\lambda(x_\lambda)
}{
P^\star(\Lambda=\lambda)
},
\]
so the corresponding fixed-pattern expectation is generally
\[
\frac{
\mathbb E_{X\sim p^\star}
[
\pi_\lambda(X_\lambda)g_\theta(X)
]
}{
P^\star(\Lambda=\lambda)
},
\]
which need not equal
$\mathcal L_{\mathrm{FM}}(\theta)$.

Nevertheless, after averaging over the random pattern,
\[
\sum_\lambda
\pi_\lambda(X_\lambda)
=
1
\qquad
p^\star\text{-a.s.},
\]
and the pattern-specific tilts cancel, yielding
\[
\mathcal L_{\mathrm{cond}}(\theta)
=
\mathcal L_{\mathrm{FM}}(\theta).
\]

This distinction concerns the ordinary-conditional surrogate in
\eqref{eq:conditional-surrogate}. The general oracle identity of
Theorem~\ref{thm:posterior-equivalence} does not rely on either argument:
it follows directly from conditioning on the full observed object
$O=(\Lambda,X_\Lambda)$.

\subsection{Identification and Estimation under Each Mechanism}
\label{app:identifiability}

The equality of population objectives must be distinguished from the
statistical problem of estimating the required completion law.

\paragraph{MCAR.}
Under MCAR,
\[
r^\star_\lambda(\cdot\mid x_\lambda)
=
p^\star_{X_{-\lambda}\mid X_\lambda}
(\cdot\mid x_\lambda),
\]
and the pattern itself carries no information about the data values.
This removes selection bias due to conditioning on
$\Lambda=\lambda$. It does not, however, guarantee unrestricted
nonparametric identification of the full joint distribution under an
arbitrary collection of missingness patterns. For example,
\[
P^\star(\Lambda=[d])>0
\]
is a simple sufficient condition because complete cases are then an
unbiased sample from $p^\star$. Weaker pattern coverage can suffice
under additional structural or parametric assumptions.

\paragraph{MAR.}
Under MAR the oracle posterior still satisfies
\[
r^\star_\lambda(\cdot\mid x_\lambda)
=
p^\star_{X_{-\lambda}\mid X_\lambda}
(\cdot\mid x_\lambda),
\]
but this identity alone does not imply that the conditional can be
estimated without accounting for how the observed records were
selected.

Indeed, conditional on an original pattern $\Lambda=\lambda$,
\[
p(x_\lambda\mid\Lambda=\lambda)
\propto
p^\star_{X_\lambda}(x_\lambda)
\pi_\lambda(x_\lambda).
\]
Consequently, the self-supervised masking procedure used to train
$q_\phi$ can target the propensity-tilted conditional characterized in
Proposition~\ref{prop:mar-ssl-tilt}. When the pattern propensity is
identified and adequate positivity holds, the stabilized weighting
scheme of Proposition~\ref{prop:ipw-ssl} removes this selection tilt.

Selection correction is distinct from pattern coverage: IPW cannot
identify a conditional involving coordinates that are never
sufficiently observed together. Identification may therefore require
additional overlap across patterns or structural parameter sharing in
the completion model.

\paragraph{MNAR.}
Under MNAR, Lemma~\ref{lem:bayes} gives
\[
r^\star_\lambda(x_{-\lambda}\mid x_\lambda)
=
p^\star_{X_{-\lambda}\mid X_\lambda}
(x_{-\lambda}\mid x_\lambda)
\frac{
\pi_\lambda(x_\lambda,x_{-\lambda})
}{
\bar\pi_\lambda(x_\lambda)
}.
\]
The numerator depends on unobserved values, so the required posterior is
generally not identifiable from the observed-data distribution without
additional assumptions or information, such as a specified missingness
model, instrumental variables, shadow variables, or other identifying
structure \citep{ipsen2020not}.

Thus the oracle result remains valid under MNAR; what fails in general
is identification of the oracle posterior from observed data alone.

\subsection{The Objective Gap under Deterministic Imputation}
\label{app:pointimp}

Remark~\ref{rem:point-imputation} states that deterministic imputation
generally changes the population objective. The gap can be written
without imposing any missingness assumption.

Let
\[
\widehat x_{-\Lambda}(O)
\]
be any deterministic rule for filling the missing coordinates and define
\[
\mathcal L_{\mathrm{imp}}(\theta)
:=
\mathbb E_O
\left[
g_\theta
\left(
X_\Lambda,
\widehat x_{-\Lambda}(O)
\right)
\right].
\]
By Theorem~\ref{thm:posterior-equivalence},
\[
\mathcal L_{\mathrm{FM}}(\theta)
=
\mathbb E_O
\mathbb E_{
X_{-\Lambda}^\star
\sim
r^\star_\Lambda(\cdot\mid X_\Lambda)
}
\left[
g_\theta(X_\Lambda,X_{-\Lambda}^\star)
\right].
\]
Therefore
\begin{equation}
\label{eq:deterministic-gap}
\boxed{
\begin{aligned}
&
\mathcal L_{\mathrm{imp}}(\theta)
-
\mathcal L_{\mathrm{FM}}(\theta)
\\
&=
\mathbb E_O
\Bigg[
g_\theta
\left(
X_\Lambda,
\widehat x_{-\Lambda}(O)
\right)
-
\mathbb E_{
X_{-\Lambda}^\star
\sim
r^\star_\Lambda(\cdot\mid X_\Lambda)
}
\left[
g_\theta(X_\Lambda,X_{-\Lambda}^\star)
\right]
\Bigg].
\end{aligned}
}
\end{equation}

Under MCAR or MAR,
$r^\star_\Lambda$ may be replaced in
\eqref{eq:deterministic-gap} by the ordinary conditional
\[
p^\star_{X_{-\Lambda}\mid X_\Lambda}
(\cdot\mid X_\Lambda).
\]

Suppose the posterior conditional mean exists and the imputation rule is
\[
\widehat x_{-\Lambda}(O)
=
\mathbb E
\left[
X_{-\Lambda}\mid O
\right].
\]
If, for $P^\star_O$-almost every $O$, the function
\[
u
\mapsto
g_\theta(X_\Lambda,u)
\]
is affine on the convex hull of the support of
$r^\star_\Lambda(\cdot\mid X_\Lambda)$, then
\[
g_\theta
\left(
X_\Lambda,
\mathbb E[X_{-\Lambda}\mid O]
\right)
=
\mathbb E
\left[
g_\theta(X_\Lambda,X_{-\Lambda})
\mid O
\right],
\]
and hence
\[
\mathcal L_{\mathrm{imp}}(\theta)
=
\mathcal L_{\mathrm{FM}}(\theta).
\]

For a general nonlinear $g_\theta$, however, this equality need not
hold. No fixed sign of the objective gap is implied because
$g_\theta(X_\Lambda,\cdot)$ need not be convex or concave.

Deterministic imputation also replaces the conditional distribution of
the missing block supplied to the training objective by the point mass
\[
\delta_{\widehat x_{-\Lambda}(O)}.
\]
Thus its conditional \emph{training-endpoint} variance is zero by
construction. This fact alone does not imply that the trained generative
model has zero conditional variance. The experiments in
Appendix~\ref{app:exp:strategies} instead measure the resulting generated
spread empirically and show that deterministic imputation substantially
reduces it relative to stochastic completion. Figure~\ref{fig:point-imputation} illustrates the two cases.

\begin{figure}[t]
\centering
\begin{tikzpicture}[x=0.46cm,y=0.46cm,>=Latex,every node/.style={font=\small}]

\node[font=\bfseries\scriptsize] at (3.1,5.3) {Point imputation};
\draw[->,gray!65] (0.4,0.7) -- (5.8,0.7)
    node[right,font=\scriptsize] {$x_{-\Lambda}$};
\draw[gray!60,thick,dashed,domain=0.9:5.2,samples=80,smooth]
    plot (\x,{0.75 + 2.9*exp(-0.9*(\x-3.0)^2)});
\draw[red!75!black,line width=1.4pt,->] (3.0,0.72) -- (3.0,4.25);
\fill[red!75!black] (3.0,0.72) circle (2.1pt);
\node[red!75!black,font=\bfseries\scriptsize] at (3.0,4.65) {one value};
\node[red!75!black,font=\scriptsize]          at (3.0,0.05) {$\hat x_{-\Lambda}=\mu$};
\node[red!75!black,font=\bfseries\scriptsize] at (3.0,-0.85) {spread removed};

\node[font=\bfseries\scriptsize] at (11.5,5.3) {Stochastic completion};
\draw[->,gray!65] (8.8,0.7) -- (14.2,0.7)
    node[right,font=\scriptsize] {$x_{-\Lambda}$};
\foreach \x in {10.0,10.4,10.8,11.2,11.5,11.85,12.25,12.8}{
    \draw[blue!40,line width=0.7pt]
        (\x,0.72) -- (\x,{0.75+2.9*exp(-0.9*(\x-11.4)^2)});
    \fill[blue!70!black]
        (\x,{0.75+2.9*exp(-0.9*(\x-11.4)^2)}) circle (1.5pt);}
\draw[gray!60,thick,dashed,domain=9.3:13.6,samples=80,smooth]
    plot (\x,{0.75 + 2.9*exp(-0.9*(\x-11.4)^2)});
\node[blue!70!black,font=\bfseries\scriptsize] at (11.4,4.65)
    {draws from $r^\star_\Lambda$};
\node[blue!70!black,font=\bfseries\scriptsize] at (11.4,-0.85)
    {spread retained};

\node[gray!60,font=\scriptsize] at (7.3,-2.0)
    {dashed curve: true conditional $r^\star_\Lambda(\cdot\mid X_\Lambda)$};

\end{tikzpicture}

\caption{Deterministic imputation versus stochastic completion, for a fixed
observed part $X_\Lambda$. A deterministic rule supplies the single value
$\hat x_{-\Lambda}(O)$, so the conditional law of the missing block entering
the training objective is the point mass
$\delta_{\hat x_{-\Lambda}(O)}$ and its conditional
\emph{training-endpoint} variance is zero by construction. Drawing from
$r^\star_\Lambda$ instead supplies a genuine sample and retains that spread.
As the text above notes, zero training-endpoint variance does not by itself
imply that the trained generative model has zero conditional variance;
Appendix~\ref{app:exp:strategies} measures the generated spread empirically.}
\label{fig:point-imputation}
\end{figure}

\section{Experimental Details and Additional Results}
\label{app:experiments}

Here we give the full protocol behind the experiments and report additional
results.

\subsection{The Training Procedure}
\label{app:algorithm}

Algorithm~\ref{alg:mdfm} gives the flow-training procedure for a fitted
completion model $q_\phi$. How $q_\phi$ is estimated is a separate statistical
problem and depends on the missingness mechanism, as discussed in
Appendix~\ref{app:mechanisms}. The learned-completion concentration analysis
conditions on a fixed $q_\phi$ fitted independently of the evaluation sample.
This independence is an analytical assumption, not a requirement of the
algorithm.

\begin{algorithm}[!ht]
\caption{Missing-Data Flow Matching}
\label{alg:mdfm}
\begin{algorithmic}[1]
\REQUIRE Incomplete data
$\{(X_{\Lambda_i}^{(i)},\Lambda_i)\}_{i=1}^n$,
fitted completion model $q_\phi$, base distribution $p_0$,
vector field $v_\theta$, completions $K$, steps $S$, learning rate $\eta$
\STATE Initialize $\theta$
\FOR{step $=1$ to $S$}
    \STATE Sample a minibatch $\mathcal B$
    \FOR{each $i\in\mathcal B$ and $k=1,\ldots,K$}
        \STATE Sample
        $\widetilde X_{-\Lambda_i}^{(i,k)}
        \sim q_\phi(\cdot\mid X_{\Lambda_i}^{(i)},\Lambda_i)$
        \STATE Set
        $\widetilde X_1^{(i,k)}
        \leftarrow
        (X_{\Lambda_i}^{(i)},\widetilde X_{-\Lambda_i}^{(i,k)})$
        \STATE Sample
        $X_0^{(i,k)}\sim p_0$ and
        $T^{(i,k)}\sim\mathrm{Unif}[0,1]$
        \STATE Set
        $X_t^{(i,k)}
        \leftarrow
        (1-T^{(i,k)})X_0^{(i,k)}
        +T^{(i,k)}\widetilde X_1^{(i,k)}$
        \STATE Compute
        $\ell^{(i,k)}
        \leftarrow
        \ell_\theta
        (X_0^{(i,k)},\widetilde X_1^{(i,k)},T^{(i,k)})$
    \ENDFOR
    \STATE
    $\widehat{\mathcal L}_{\mathcal B,K}^{\phi}
    \leftarrow
    \frac{1}{|\mathcal B|}
    \sum_{i\in\mathcal B}
    \frac{1}{K}
    \sum_{k=1}^{K}
    \ell^{(i,k)}$
    \STATE
    $\theta
    \leftarrow
    \theta-\eta
    \nabla_\theta
    \widehat{\mathcal L}_{\mathcal B,K}^{\phi}$
\ENDFOR
\ENSURE Trained vector field $v_\theta$
\end{algorithmic}
\end{algorithm}

\paragraph{Computational cost.}
For a minibatch of size $B$ and $K$ completions per observation, each training
step uses $BK$ completion draws and $BK$ flow-matching loss evaluations, so the
completion-dependent computation scales linearly in $K$. Materializing all
completions requires $O(BKd)$ auxiliary storage, while processing completions
sequentially reduces this to $O(Bd)$. Section~\ref{sec:analysis} analyzes the
corresponding statistical allocation through $n$, $K$, and the fixed evaluation
budget $N=nK$.

\subsection{Targets, Architecture, and Estimation Procedure}
\label{app:exp:protocol}

\paragraph{Synthetic target.} The synthetic runs use a zero-mean Gaussian whose
covariance is an AR(1) correlation matrix with parameter $0.6$. This gives real
cross-coordinate dependence while keeping the conditional
$p^\star(x_{-\Lambda}\mid x_\Lambda)$ available in closed form, so the oracle
draws come from the true conditional rather than from a learned stand-in. The
bounded-support runs used for Theorem~\ref{thm:completion-variance} replace it with the same
Gaussian truncated to $[-R,R]^d$ by rejection, which is what
Assumption~\ref{ass:bdd} requires. We sweep $d\in\{10,50,100\}$ and
$\rho\in\{0.1,0.3,0.5,0.7,0.9\}$ under MCAR, and we repeat every setting over
five seeds.

\paragraph{Fixed-parameter design.} All estimator-level quantities are measured
at frozen $\theta$, which is what the finite-sample theory describes. We evaluate
at a random initialization and at a trained checkpoint so the conclusions do not
depend on one point in parameter space. The variance components $\tau_\theta^2$,
$\sigma_{\theta,\mathrm{base}}^2$, and $\sigma_{\theta,\mathrm{miss}}^2$ are
estimated by nested Monte Carlo with an outer loop over observations and an inner
loop over completions, bias corrected for the finite inner sample. The empirical
variance of $\widehat{\mathcal L}^\star_{n,K}$ is measured by directly simulating
the estimator many times, so the formula and the measurement come from two
independent estimation paths and their agreement is not circular.

\paragraph{Missingness patterns.} The estimator-level experiments draw an
independent Bernoulli mask per coordinate. The training experiments instead draw
a fixed set of sixteen patterns once and assign each row one of them uniformly,
which keeps the mechanism exactly MCAR while bounding the number of distinct
Gaussian conditionings a minibatch can contain. This also mirrors real tabular
data, where missingness usually arrives in a limited number of recurring shapes.

\paragraph{Real data.} We use seven numeric UCI datasets, namely concrete, wine,
diabetes, breast cancer, ionosphere, sonar, and digits, ranging from $178$ to
$1797$ rows and from $9$ to $61$ features after constant columns are dropped.
Each dataset is standardized on the training split. Entries are hidden
completely at random at five rates, with the mask repaired so that no row and no
column is fully missing. The hidden values are held aside only to score
imputation and are never seen by any method. Generative quality is measured
against the held-out test split, which is complete.

\paragraph{Architecture and optimization.} Every method shares one vector-field
architecture, a multilayer perceptron with sinusoidal time embedding, three
hidden layers of width $128$, and SiLU activations. The completion model
$q_\phi$ used on real data is a conditional flow matching model with the same
backbone that predicts a velocity on the missing coordinates given the observed
values and the mask. It is trained by self supervised masking, hiding a further
random subset of the observed entries and predicting them from the rest, so it
never touches a value that the missingness hid. All models use Adam \citep{kingma2014adam}.

\paragraph{Synthetic sample sizes and sweeps.}
The oracle equivalence runs estimate the complete data objective from $8192$
endpoints with $256$ inner draws of the base point and the time at each
endpoint. Loss gaps average $400$ independent estimates built from $512$
observations each. Gradient cosines use minibatches of $256$, $1024$, $4096$,
and $16384$, and the paper reports the largest. The variance runs sweep
$d\in\{10,50,100\}$, $\rho\in\{0.1,0.3,0.5,0.7,0.9\}$, and
$K\in\{1,2,4,8,16,32\}$. The bounded support runs that accompany Theorem~3
truncate the Gaussian to $[-R,R]^d$ with $R=4$, which covers almost all of the
standardised mass while giving Assumption~2 a concrete constant.

\paragraph{Nested Monte Carlo and direct simulation.}
The three variance components are estimated by a nested design. The outer loop
draws $2048$ incomplete observations. The inner loop draws $48$ completions for
each observation, and each completed endpoint is scored with $256$ draws of the
base point and the time. The inner sample is finite, so the raw within
observation variance is biased upward by the base Monte Carlo term divided by
the number of inner draws, and we subtract that term before reporting. The
estimator variance is measured separately by simulating
$\widehat{\mathcal L}^\star_{n,K}$ directly, using $1500$ independent
simulations at $n=512$. The formula and the simulation therefore never share a
computation, which is what makes their agreement a check rather than a
restatement.

\paragraph{Fixed budget experiment.}
The fixed budget runs hold $N=nK=1024$ and sweep $K\in\{1,2,4,8,16\}$, so $n$
falls from $1024$ to $64$ as $K$ grows. Each point averages $600$ independent
simulations of the estimator. The three variance components are estimated once
per seed and reused at every $K$, which is what makes the plotted line affine
and nondecreasing in $K$ as Corollary~\ref{cor:variance-consequences} predicts.

\paragraph{Estimating the Lipschitz constant.}
We estimate $\widehat L_g$ from sampled endpoints. We draw $4096$ endpoints from
the target, average the gradient of the endpoint loss with respect to the
endpoint over $64$ draws of the base point and the time, take the Euclidean norm
of that averaged gradient at each endpoint, and report the maximum over
endpoints. Averaging before taking the norm removes most of the Monte Carlo
noise that would otherwise inflate a maximum. This is a sample based estimate
over a finite set of endpoints rather than a global Lipschitz constant, so every
quantity built from it is an empirical proxy and not a certified bound.

\paragraph{What the error bars show.}
Every reported number averages five independent runs with seeds $0$ through $4$.
A seed fixes the data draw, the missingness mask, the split, and the network
initialisation. Error bars in all figures are the standard error of the mean
across those five seeds, so they capture seed to seed variation and not the
Monte Carlo error inside a single run. Table~2 and Table~3 report means without
error bars because the seed to seed spread there is smaller than the reported
precision.

\paragraph{Unit of replication.}
The five seeds, rather than the individual Monte Carlo draws within a
seed, are the independent experimental replicates. Inner completion,
endpoint, base-point, and time draws reduce the numerical error of an
estimate within one seed, but they are not counted as additional
independent runs. This is why uncertainty is computed from variation
across seeds only.

\paragraph{Splits and the completion model.}
Real datasets are split into $80$ percent training and $20$ percent test, drawn
at random per seed, with no separate validation split, since no per dataset
tuning takes place. Standardisation uses the training mean and standard
deviation only. Missing entries are created in the training split alone, and the
test split stays complete and is the reference for the generative metrics. On
the synthetic runs the fitted completion models use a genuinely independent
auxiliary sample, namely fresh rows drawn from the target and masked
independently. On the real datasets we fit $q_\phi$ on the same incomplete
training split that trains the flow, because holding out a second sample from
datasets this small would weaken both stages. This departs from the independence condition used in
Corollary~\ref{cor:learned-total}, and it is the limitation recorded
in the conclusion.

\paragraph{Optimisation and architecture settings.}
Every vector field is a multilayer perceptron with three hidden layers of width
$128$ and SiLU activations. The time embedding is sinusoidal with $8$
frequencies, giving $16$ features, which are concatenated with the state. The
conditional completion model shares that backbone and additionally takes the
observed values and the mask, so its input width is $3d+16$. All models use Adam
with learning rate $2\times10^{-3}$, default moments, and no weight decay. The
frozen parameter checkpoints are trained for $2000$ steps at batch size $1024$.
On real data the flow trains for $4000$ steps and the completion model for
$3000$ steps, both at batch size $256$. Self supervised masking hides $30$
percent of the observed entries at each step.

\paragraph{Generation and evaluation.}
Samples are generated by integrating the learned field from the base
distribution to the data with an explicit Euler solver using $100$ uniform
steps. Inside the training loop, where completions are drawn every step, the
completion sampler uses $20$ steps for speed. Each evaluation draws $4096$
generated samples and compares them against the held out test split.

\paragraph{Datasets after preprocessing.}
We use the standard numeric versions of seven public tabular datasets. After
dropping constant columns the dimensions are concrete with $1030$ rows and $9$
features, wine with $178$ rows and $13$ features, diabetes with $442$ rows and
$10$ features, breast cancer with $569$ rows and $30$ features, ionosphere with
$351$ rows and $34$ features after one constant column is removed, sonar with
$208$ rows and $61$ features, and digits with $1797$ rows and $61$ features
after three constant pixels are removed. Only feature columns are used, and
target columns are discarded, since the task is joint generative modelling
rather than prediction.

\paragraph{Baseline implementations.}
Mean imputation fills each column with its observed mean. MICE is iterative
imputation with a Bayesian ridge predictor, $10$ rounds, and posterior sampling
switched on so completions are stochastic. MissForest is the same iterative
scheme with a random forest predictor of $64$ trees and $8$ rounds. Both use the
standard scikit-learn iterative imputer. GAIN follows the original method with a
generator and a mask predicting discriminator that both have two hidden layers
of width $d$, a hint rate of $0.9$, a reconstruction weight of $10$, $2000$
steps, batch size $128$, and Adam with learning rate $10^{-3}$. GAIN scales each
column to the unit interval internally, which its sigmoid output layer assumes,
and the outputs are mapped back before use.

\paragraph{Metric definitions.}
Sliced Wasserstein-2 draws $256$ random unit directions, projects both samples
onto each direction, sorts, takes the root mean squared difference of the sorted
values, and averages over directions. Energy distance is
$2\,\mathbb E\|A-B\| - \mathbb E\|A-A'\| - \mathbb E\|B-B'\|$ estimated from at
most $2048$ points per sample. Maximum mean discrepancy is the squared form with
a radial basis kernel whose bandwidth is the median pairwise distance of the
pooled sample, again capped at $2048$ points. Covariance error is the Frobenius
norm of the difference between the generated and reference covariance matrices
divided by the Frobenius norm of the reference. The conditional spread diagnostic
regresses one coordinate on all the others by least squares and reports the
standard deviation of the residual, which for a Gaussian is exactly the
conditional standard deviation of that coordinate.

\paragraph{Mechanism formulas and strength.}
The three mechanisms share one logistic form. Coordinate $j$ of row $x$ is
observed with probability $\sigma(b + a\,z_j)$, where $\sigma$ is the logistic
function, $b=\mathrm{logit}(1-\rho)$ centres the average missing fraction at
$\rho$, and $a$ is the strength reported in Table~5. Under MCAR we set $z_j=0$
for every coordinate. Under MAR we set $z_j=x^1$ for every coordinate, so the
probability depends only on the first coordinate. Under MNAR we set $z_j=x^j$,
so the probability depends on the value of the coordinate itself. The first
coordinate is always observed, which is what makes the MAR case genuinely MAR
rather than accidentally MNAR, and any row that would be fully missing has its
first coordinate restored. A strength of zero recovers MCAR, and larger
strengths make the mask depend more sharply on the data, so the strengths $1$,
$2$, and $3$ in Table~5 are the coefficient $a$ in that logit.

\paragraph{Software and hardware.}
All experiments run in python 3.10 with pytorch 2.0, numpy 1.24, scikit-learn
1.3, and matplotlib 3.7 on Linux. Each run used one NVIDIA GPU with 12 GB of
memory together with 8 CPU cores and 16 GB of system memory. Every experiment
also runs on CPU and gives the same numbers, since all random draws use explicit
generator objects rather than global state. The longest single run finishes in
under an hour and the full set of results takes a few GPU hours. 

Six of the seven datasets are public benchmarks from the UCI Machine
Learning Repository \\(\url{https://archive.ics.uci.edu}) distributed under
the Creative Commons Attribution 4.0 International license; the diabetes
data is the regression dataset of Efron et al. (2004) distributed with
scikit-learn. PyTorch, scikit-learn, NumPy, and Matplotlib are released
under BSD or BSD-compatible licenses.

\subsection{MAR Self-Supervision Bias}
\label{app:exp:mar-ssl}

We construct a controlled MAR example in which the oracle completion law is
ignorable but naive self-supervised training is selection biased.

Let $X=(X^1,X^2,X^3)$ follow a zero-mean Gaussian distribution with AR(1)
correlation parameter $0.6$. Consider the three observed patterns
\[
\lambda_a=\{1,3\},
\qquad
\lambda_b=\{1,2\},
\qquad
\lambda_c=\{1,2,3\}.
\]
For strength $a\in\{0,0.05,0.10,0.15,0.20,0.25\}$, define
\[
\pi_{\lambda_a}(x)=0.25,
\]
\[
\pi_{\lambda_b}(x)=0.35-a\tanh(x^2),
\qquad
\pi_{\lambda_c}(x)=0.40+a\tanh(x^2).
\]
These probabilities are positive and sum to one. Each propensity depends only
on coordinates observed under its corresponding pattern, so the mechanism is
MAR.

Pattern $\lambda_a$ is missing coordinate $2$, so its oracle completion law is
\[
p^\star(X^2\mid X^1,X^3).
\]
We estimate this conditional from complete-pattern observations
$\Lambda=\lambda_c$ by artificially masking coordinate $2$. Because
$\pi_{\lambda_c}$ depends on $X^2$, Proposition~\ref{prop:mar-ssl-tilt}
predicts that naive self-supervision targets a propensity-tilted conditional.

We compare three estimators. The first uses naive unweighted self-supervision.
The second uses stabilized inverse-probability weights
\[
w_{\lambda_c}(x)
=
\frac{
P^\star(\Lambda=\lambda_c)
}{
\pi_{\lambda_c}(x)
}.
\]
The third is a mask-compatible control obtained by replacing the dependence on
$x^2$ in the two nonconstant propensities by dependence on the retained
coordinate $x^1$. In this control, the propensity is constant in the
artificially masked coordinate after conditioning on the retained variables.

For all three cases, we fit the same Gaussian conditional model for
$X^2\mid(X^1,X^3)$. We evaluate on an independent complete test sample using
the expected conditional $W_1$ error and the relative flow-matching objective
bias. For the latter, all other missing patterns use oracle completion so that
the measured bias isolates this self-supervision effect. Results are averaged
over five seeds.

Figure~\ref{fig:mar-ssl} reports the resulting completion and flow-matching
errors. At $a=0$, the missingness mechanism has no selection tilt and all three
procedures perform similarly. As $a$ increases, the conditional $W_1$ error of
naive self-supervision grows steadily, reaching approximately $0.19$ at
$a=0.25$. The mask-compatible and stabilized-IPW procedures remain near zero
throughout the sweep.

The downstream objective shows the same qualitative behavior. The relative
bias of naive self-supervision increases from the Monte Carlo floor to roughly
$0.10\%$ at the strongest MAR setting, while the other two procedures remain
below $0.01\%$. Thus the experiment isolates the selection effect in
Proposition~\ref{prop:mar-ssl-tilt}: the bias appears when the propensity
depends on the artificially hidden variable and disappears when that dependence
is removed or corrected by stabilized inverse-probability weighting.

All values are averaged over five seeds. Error bars report standard errors
across seeds.

\subsection{Oracle Equivalence}
\label{app:exp:equiv}

Table~\ref{tab:equiv} reports the loss gap and the gradient agreement behind the
first claim of the experiments section. The gap shrinks with dimension and the
MDFM-to-FM gradient cosine tracks the complete-data same-noise baseline at every
dimension, so the residual is Monte Carlo noise and not a systematic tilt.

\begin{table}[h]
\centering
\small
\setlength{\tabcolsep}{5pt}
\begin{tabular}{@{}rrccc@{}}
\toprule
$d$ & $K$ & $\dfrac{|\mathcal L_{\mathrm{MDFM}}-\mathcal L_{\mathrm{FM}}|}{\mathcal L_{\mathrm{FM}}}$ & $\cos(\mathrm{M},\mathrm{F})$ & $\cos(\mathrm{F},\mathrm{F})$ \\
\midrule
10  & 1 & $0.48\%$ & $0.901$ & $0.894$ \\
10  & 4 & $0.43\%$ & $0.931$ & $0.894$ \\
50  & 1 & $0.26\%$ & $0.967$ & $0.964$ \\
50  & 4 & $0.24\%$ & $0.977$ & $0.964$ \\
100 & 1 & $0.10\%$ & $0.967$ & $0.966$ \\
100 & 4 & $0.10\%$ & $0.976$ & $0.966$ \\
\bottomrule
\end{tabular}
\caption{Oracle equivalence at fixed parameters. $\cos(\mathrm{M},\mathrm{F})$ is
the cosine between the MDFM and complete-data gradient estimates and
$\cos(\mathrm{F},\mathrm{F})$ is the cosine between two independent complete-data
estimates, which is the value Monte Carlo noise alone produces. Mean over five
seeds and missing rates $\rho\in\{0.3,0.7\}$, pooling the random-initialization
and trained checkpoints. Cosines use the largest batch.}
\label{tab:equiv}
\end{table}

\subsection{Variance Decomposition and the $K=1$ Identity}
\label{app:exp:variance}

Figure~\ref{fig:theory} in Section~\ref{sec:analysis} numerically verifies
Theorem~\ref{thm:variance-decomposition} and
Corollary~\ref{cor:variance-consequences}. The variance components and the
estimator variance are measured by the two independent routes described in
Appendix~\ref{app:exp:protocol}, nested Monte Carlo for the former and direct
simulation of the estimator for the latter.

\subsection{The Missing-Fraction Bound}
\label{app:exp:bound}

Figure~\ref{fig:bound} plots the measured
$\sigma_{\theta,\mathrm{miss}}^2$ against $\rho d$, together with the
empirical proxy
$R^2\widehat L_g^{\,2}\rho d$ for the right-hand side of
Theorem~\ref{thm:completion-variance}. Here, $\widehat L_g$ is estimated from sampled
pairs rather than obtained as a certified global Lipschitz constant, so the
dashed curves are not formal upper bounds. The measured variance grows
approximately linearly in $\rho d$, which is empirically consistent with the
dependence appearing in Theorem~\ref{thm:completion-variance}; the theorem itself asserts
only an inequality and does not claim equality or linear scaling. The ratio
\[
\frac{\sigma_{\theta,\mathrm{miss}}^2}
{R^2\widehat L_g^{\,2}\rho d}
\]
is on the order of $10^{-4}$ and decreases from approximately
$5\times 10^{-4}$ at $d=10$ to approximately $1\times 10^{-4}$ at
$d=100$. Thus, the empirical proxy exceeds the measured variance by three
to four orders of magnitude in these experiments. The ratio decreases with
dimension because the estimated Lipschitz constant $\widehat L_g$ grows
faster with dimension than the measured variance. Finally, the curves for
the three dimensions nearly collapse when plotted against $\rho d$, which
supports using $\rho d$, rather than $d$ alone, as the missingness quantity
appearing in the bound.

\begin{figure}[!ht]
\centering
\includegraphics[width=\linewidth]{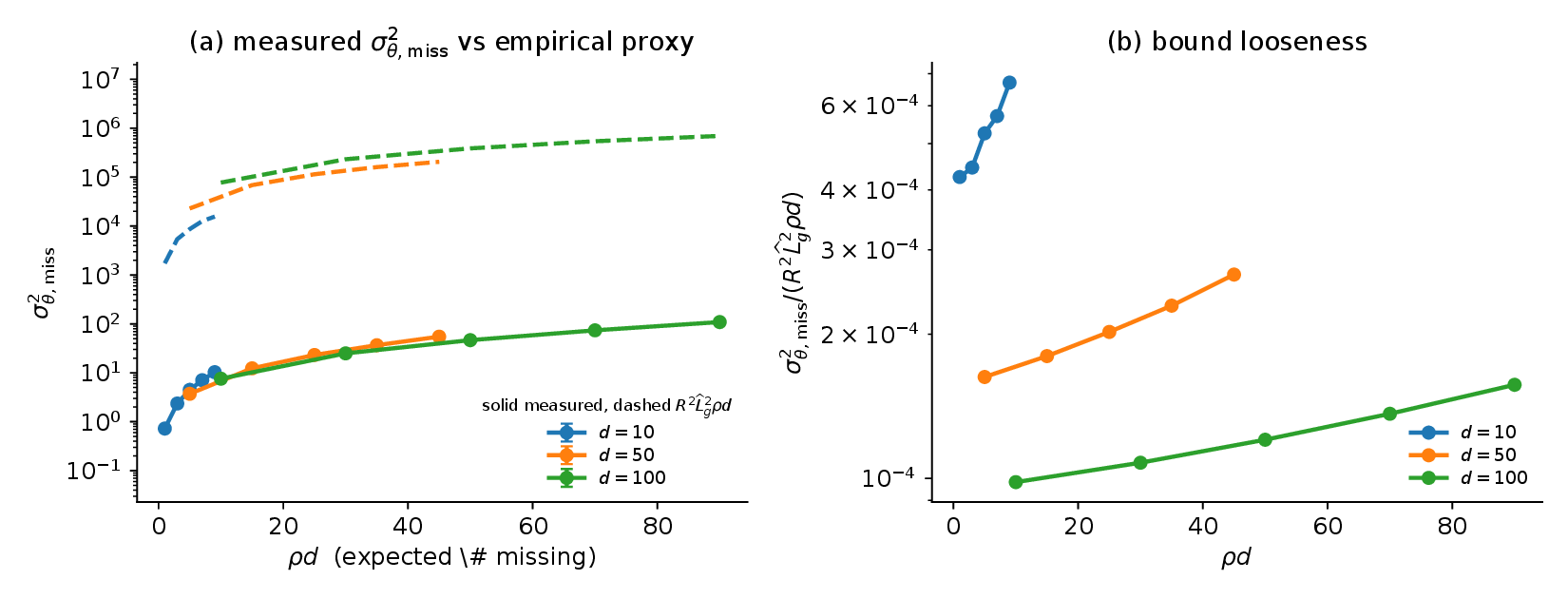}
\caption{The missing-fraction bound. Left, the measured $\sigma^2_{\theta,\mathrm{miss}}$ against the expected number of missing coordinates $\rho d$, with $R^2\widehat{L}_g^{2}\rho d$ as the dashed line of matching color. Here $\widehat{L}_g$ is estimated from sampled pairs rather than derived as a global Lipschitz constant, so the dashed curves are an empirical proxy for the quantity in Theorem~3 and are not certified upper bounds. Right, the ratio of the two, on the order of $10^{-4}$ across dimensions.}
\label{fig:bound}
\end{figure}

\subsection{Completion Models and the Transport Bound}
\label{app:exp:certificate}

Table~\ref{tab:bias} lists the twelve completion models behind
Figure~\ref{fig:bias} together with their completion distance, the bound, and the ratio of the two. The mean-shift and conditional-mean families have a
closed-form $W_1$, since translation and point-mass couplings are optimal, so those rows use an exact $W_1$. The remaining rows report the
cost of one explicit coupling, which is an upper bound on $W_1$. The oracle row has zero distance, so its
measured bias of $0.10$ is the Monte Carlo noise floor against which every other
row should be read. Throughout, $L_g$ is estimated from sampled pairs and the completion models are
Gaussian rather than compactly supported, so these plots probe the inequality of
Lemma~\ref{lem:completion-bias} numerically rather than under its literal hypotheses.

\begin{table}[!ht]
\centering
\small
\setlength{\tabcolsep}{4pt}
\begin{tabular}{@{}lcccc@{}}
\toprule
completion model & $W_1$ & $|\mathcal L_\phi-\mathcal L_{\mathrm{FM}}|$ & $\mathbb E_O[W_1]$ & ratio \\
\midrule
oracle            & exact & 0.101 & 0.000 & -- \\
mean shift $0.05$ & exact & 0.096 & 0.140 & 0.057 \\
mean shift $0.1$  & exact & 0.118 & 0.279 & 0.035 \\
fitted, $n=4096$  & bound & 0.065 & 0.321 & 0.019 \\
mean shift $0.2$  & exact & 0.341 & 0.559 & 0.049 \\
variance $\times1.5$ & bound & 5.760 & 1.037 & 0.449 \\
variance $\times0.5$ & bound & 3.391 & 1.039 & 0.264 \\
mean shift $0.4$  & exact & 1.072 & 1.090 & 0.081 \\
marginal          & bound & 3.549 & 1.843 & 0.158 \\
fitted, $n=256$   & bound & 1.671 & 1.904 & 0.065 \\
conditional mean  & exact & 4.338 & 2.022 & 0.178 \\
mean shift $0.8$  & exact & 4.362 & 2.237 & 0.159 \\
\bottomrule
\end{tabular}
\caption{Completion models ordered by distance to the truth. The ratio is
$|\mathcal L_\phi-\mathcal L_{\mathrm{FM}}|$ divided by $L_g\,\mathbb E_O[W_1]$ from Lemma~\ref{lem:completion-bias} and is below one everywhere.
Rows marked exact have a closed-form $W_1$.}
\label{tab:bias}
\end{table}

Figure~\ref{fig:bias} in Section~\ref{sec:analysis} visualizes these results
and compares the measured objective bias with the transport bound in
Lemma~\ref{lem:completion-bias}.

\subsection{Trained Models Under Different Completion Strategies}
\label{app:exp:strategies}

Figure~\ref{fig:resamp} in Section~\ref{sec:experiments} summarizes the main
comparison, while Table~\ref{tab:trained} reports the numerical values. All
seven strategies use the same incomplete observations and masks, so they differ
only in how the missing coordinates are supplied.

The stochastic strategies are within seed-level variation of one another,
including the complete-data reference. Conditional-mean imputation is the only
strategy with substantially reduced conditional spread.

\begin{table}[!ht]
\centering
\small
\setlength{\tabcolsep}{5pt}
\begin{tabular}{@{}lccc@{}}
\toprule
strategy & cov.\ err & sliced $W_2$ & cond.\ s.d.\ ratio \\
\midrule
complete data              & 0.123 & 0.050 & 0.97 \\
oracle resampled $K{=}4$   & 0.079 & 0.040 & 0.99 \\
oracle resampled $K{=}1$   & 0.106 & 0.053 & 0.98 \\
oracle frozen              & 0.091 & 0.044 & 0.98 \\
fitted resampled $K{=}1$   & 0.100 & 0.048 & 0.97 \\
fitted frozen              & 0.104 & 0.047 & 0.97 \\
conditional mean           & 0.250 & 0.164 & \textbf{0.64} \\
\bottomrule
\end{tabular}
\caption{Trained-model quality under different completion strategies on the
synthetic target. Fitted models are Gaussian conditionals estimated only from
incomplete data. Every stochastic strategy recovers the conditional spread.
Deterministic conditional-mean imputation alone collapses it. Mean over five
seeds.}
\label{tab:trained}
\end{table}

\subsection{Allocation Under a Fixed Budget}
\label{app:exp:budget}

Figure~\ref{fig:budget} sweeps $K$ with the total evaluation budget $nK$ held
constant. The empirical variance rises with $K$ and tracks the prediction of
Corollary~\ref{cor:variance-consequences}, so a single completion is the right choice whenever
more incomplete observations are available.


\subsection{Objective Gap Under MCAR, MAR, and MNAR}
\label{app:exp:mechanisms}

Table~\ref{tab:mech} in Section~\ref{sec:experiments} reports the objective gap
obtained by completing from the ordinary data conditional
$p^\star(x_{-\Lambda}\mid x_\Lambda)$ under MCAR, MAR, and MNAR. This is the
quantity characterized by Proposition~\ref{prop:naive}. Under MCAR and MAR,
the weight $w$ is identically one and the gap remains at the Monte Carlo noise
floor. Under MNAR, the weight is nonconstant and the gap grows with mechanism
strength, reaching nearly five percent in the strongest setting.


\subsection{Real Data, Normalized Scores}
\label{app:exp:normalized}

Table~\ref{tab:realnorm} repeats the real-data comparison using a per-cell
normalized score instead of a rank. Within each dataset, missing rate, and seed
we map the best method to zero and the worst to one and then average. This
checks that the ranking in the main text is not an artifact of rank ties, and it
also shows the size of the gaps rather than only their order. The conclusion is
the same. MDFM and MICE lead, MDFM reaches the lowest covariance error at the highest missing rate, and mean imputation and GAIN trail by a wide margin.

Figure~\ref{fig:realdata} shows where those differences come from. At the lowest
missing rate every method sits within a narrow band of the skyline, so the pooled
orderings in the main text and in Table~\ref{tab:realnorm} are driven
almost entirely by the high missingness end of the range. Mean imputation and
GAIN separate first and keep climbing, and MissForest stays close to the leaders
through the middle of the range before rising sharply at the highest rate. MICE
and MDFM remain close to each other throughout, and MDFM reaches the lowest
covariance error of any method trained on incomplete data at the highest missing
rate, which is where the dependence structure is hardest to recover. The skyline
is flat across the whole sweep, which is the sanity check that it never sees a
missing value.

\begin{table}[!ht]
\centering
\small
\setlength{\tabcolsep}{5pt}
\begin{tabular}{@{}lcccc@{}}
\toprule
method & sW$_2$ & energy & MMD & cov \\
\midrule
Complete-FM (skyline) & 0.131 & 0.084 & 0.057 & 0.099 \\
Mean + FM             & 0.713 & 0.588 & 0.881 & 0.934 \\
MICE + FM             & \textbf{0.204} & \textbf{0.123} & \textbf{0.078} & 0.334 \\
MissForest + FM       & 0.267 & 0.179 & 0.155 & 0.358 \\
GAIN + FM             & 0.816 & 0.844 & 0.758 & 0.564 \\
MDFM (ours)           & \underline{0.237} & \underline{0.134} & \underline{0.122} & \textbf{0.305} \\
\bottomrule
\end{tabular}
\caption{Real data with per-cell normalized scores, where zero is the best method
in that cell and one is the worst, averaged over datasets, missing rates, and
seeds. Best incomplete-data method in \textbf{bold}, second \underline{underlined}.}
\label{tab:realnorm}
\end{table}

\paragraph{Per-dataset results.}
The rank and score summaries pool every dataset, so they can hide a method that
wins on average while failing on a particular kind of data.
Table~\ref{tab:perdataset} therefore reports one row per dataset, with each entry
averaging the five missing rates and the five seeds. Lower is better, and
Complete-FM is shown for reference even though it never sees a missing value.
Three things stand out. MICE and MDFM take every row between them, and no other
method is best anywhere. The two are close throughout, with the largest
difference on any dataset being $0.020$ and several rows separated by $0.002$ or
less, so the split of rows between them is a weaker signal than the size of
their lead over everything else. Mean imputation and GAIN are behind on all seven
datasets, and GAIN is the weakest method on every one.
 
The distance to the skyline varies far more across datasets than the distance
between the two leading methods. The best incomplete-data method sits within
$1.4$ percent of Complete-FM on ionosphere and within $3.7$ percent on wine, but
trails by $14$ percent on diabetes and by $28$ percent on concrete. The two
widest gaps are the two lowest-dimensional datasets, which is consistent with a
completion model having fewer correlated coordinates to predict from, though with
only seven datasets we treat that as an observation rather than a trend. Sample
size does not explain the pattern, since the widest gap occurs on one of the
larger datasets and the narrowest occurs on one of the smaller ones.
 
\begin{table}[!ht]
\centering
\small
\setlength{\tabcolsep}{4pt}
\begin{tabular}{@{}lrrcccccc@{}}
\toprule
dataset & $n$ & $d$ & Complete-FM & Mean & MICE & MissForest & GAIN & MDFM \\
\midrule
concrete & 824 & 9 & 0.144 & 0.318 & \textbf{0.185} & 0.216 & 0.427 & \textbf{0.185} \\
wine & 142 & 13 & 0.347 & 0.462 & \textbf{0.360} & 0.399 & 0.552 & 0.375 \\
diabetes & 354 & 10 & 0.239 & 0.396 & 0.275 & 0.278 & 0.436 & \textbf{0.273} \\
breast cancer & 455 & 30 & 0.270 & 0.433 & \textbf{0.296} & 0.310 & 0.452 & 0.316 \\
ionosphere & 281 & 34 & 0.346 & 0.490 & 0.353 & 0.381 & 0.580 & \textbf{0.351} \\
sonar & 166 & 61 & 0.375 & 0.522 & \textbf{0.395} & 0.413 & 0.589 & 0.401 \\
digits & 1438 & 61 & 0.341 & 0.551 & 0.371 & 0.393 & 0.607 & \textbf{0.358} \\
\bottomrule
\end{tabular}
\caption{Per-dataset sliced Wasserstein distance to held-out data, one
row per dataset. Each entry averages the five missing rates and the five
seeds. Lower is better. The best method trained on incomplete data is in
\textbf{bold} on each row, and both are bold where they tie at the printed
precision. Complete-FM is shown for reference since it never sees a missing
value. Columns $n$ and $d$ give the training rows and the feature count
after constant columns are dropped.}
\label{tab:perdataset}
\end{table}

\subsection{Scope of the Baseline Comparison}
\label{app:exp:scope}

Our baselines isolate the contribution of the method rather than survey the
field, since every method shares one vector-field backbone and differs only in
how the missing coordinates are supplied. CFMI and CSDI solve a related but
different problem, producing completions rather than a joint generative model. In
our framework they are natural choices for $q_\phi$ rather than baselines for the
joint task, and Lemma~\ref{lem:completion-bias} already characterizes how any such model
affects the objective through its Wasserstein distance to the truth. HyperImpute
selects among such models, so the same reasoning applies. Plugging them into our
framework is left to future work.

\end{document}